\documentclass{article}
\usepackage[letterpaper, left=1in, right=1in, top=1in, bottom=1in]{geometry}
\usepackage{natbib}
\bibpunct{(}{)}{;}{a}{,}{,}
\usepackage[colorlinks,citecolor=blue!70!black,linkcolor=blue!70!black,
urlcolor=blue!70!black,breaklinks=true]{hyperref}
\usepackage{parskip}
\usepackage[dvipsnames, table]{xcolor}
\definecolor{lightblue}{RGB}{220,235,250}
\definecolor{sotacolor}{RGB}{198,239,206}
\newcommand{\legendpatch}[1]{\textcolor{#1}{\rule{8pt}{8pt}}}
\usepackage{microtype}
\usepackage{booktabs} 
\usepackage{amsmath}
\usepackage{dsfont} 
\usepackage{amssymb}
\usepackage{xspace}
\usepackage{url}            
\usepackage{algorithm}
\usepackage[noend]{algorithmic}
\usepackage{amsthm}
\usepackage{listings}
\usepackage{bm}
\usepackage{bbm}
\usepackage{enumitem}
\usepackage{nicefrac}       
\usepackage{mathrsfs} 
\usepackage{inconsolata}
\usepackage{wrapfig}
\usepackage{graphicx}
\usepackage{bigints}
\usepackage{transparent}
\usepackage{comment}
\usepackage[nameinlink,capitalize]{cleveref}
\makeatletter
\AddToHook{cmd/appendix/before}{
\def\cref@section@alias{appendix}
\def\cref@subsection@alias{appendix}
\def\cref@subsubsection@alias{appendix}
}
\makeatother
\usepackage{thmtools}
\usepackage{thm-restate}
\usepackage{nicematrix}
\usepackage{arydshln}
\usepackage{fancyhdr}
\usepackage{mathtools}
\usepackage[scaled=.88]{helvet}
\usepackage{breakcites}
\usepackage{multirow}
\usepackage{subcaption}
\usepackage{caption}

\DeclareMathOperator*{\argmax}{arg\,max}

\DeclareMathOperator{\unif}{{unif}}

\renewcommand{\epsilon}{\varepsilon}

\newtheoremstyle{spaced}
  {6pt}   %
  {0pt}   %
  {\itshape} %
  {}       %
  {\bfseries} %
  {.}      %
  {0.5em}  %
  {}
\theoremstyle{spaced}

\newtheorem{proposition}{Proposition}

\newtheorem{remark}{Remark}

\newcommand{\algcommentlight}[1]{\textcolor{blue!70!black}{\transparent{0.5}\small{\texttt{\textbf{//\hspace{2pt}#1}}}}}

\DeclarePairedDelimiter{\abs}{\lvert}{\rvert} %

\DeclarePairedDelimiter{\crl}{\{}{\}}
\DeclarePairedDelimiter{\prn}{(}{)}

\DeclarePairedDelimiterX{\infdiv}[2]{(}{)}{%
  #1\;\delimsize\|\;#2%
}

\def\ddefloop#1{\ifx\ddefloop#1\else\ddef{#1}\expandafter\ddefloop\fi}
\def\ddef#1{\expandafter\def\csname bb#1\endcsname{\ensuremath{\mathbb{#1}}}}
\ddefloop ABCDEFGHIJKLMNOPQRSTUVWXYZ\ddefloop
\def\ddefloop#1{\ifx\ddefloop#1\else\ddef{#1}\expandafter\ddefloop\fi}
\def\ddef#1{\expandafter\def\csname b#1\endcsname{\ensuremath{\mathbf{#1}}}}
\ddefloop ABCDEFGHIJKLMNOPQRSTUVWXYZ\ddefloop
\def\ddef#1{\expandafter\def\csname sf#1\endcsname{\ensuremath{\mathsf{#1}}}}
\ddefloop ABCDEFGHIJKLMNOPQRSTUVWXYZ\ddefloop
\def\ddef#1{\expandafter\def\csname c#1\endcsname{\ensuremath{\mathcal{#1}}}}
\ddefloop ABCDEFGHIJKLMNOPQRSTUVWXYZ\ddefloop
\def\ddef#1{\expandafter\def\csname h#1\endcsname{\ensuremath{\widehat{#1}}}}
\ddefloop ABCDEFGHIJKLMNOPQRSTUVWXYZ\ddefloop
\def\ddef#1{\expandafter\def\csname hc#1\endcsname{\ensuremath{\widehat{\mathcal{#1}}}}}
\ddefloop ABCDEFGHIJKLMNOPQRSTUVWXYZ\ddefloop
\def\ddef#1{\expandafter\def\csname t#1\endcsname{\ensuremath{\widetilde{#1}}}}
\ddefloop ABCDEFGHIJKLMNOPQRSTUVWXYZ\ddefloop
\def\ddef#1{\expandafter\def\csname tc#1\endcsname{\ensuremath{\widetilde{\mathcal{#1}}}}}
\ddefloop ABCDEFGHIJKLMNOPQRSTUVWXYZ\ddefloop
\def\ddefloop#1{\ifx\ddefloop#1\else\ddef{#1}\expandafter\ddefloop\fi}
\def\ddef#1{\expandafter\def\csname scr#1\endcsname{\ensuremath{\mathscr{#1}}}}
\ddefloop ABCDEFGHIJKLMNOPQRSTUVWXYZ\ddefloop

\let\oldparagraph\paragraph
\renewcommand{\paragraph}[1]{\oldparagraph{#1}}

\renewcommand{\epsilon}{\varepsilon}

\newcommand{\ldef}{\vcentcolon=}

\renewcommand{\bigm}[1]{%
  \ifcsname fenced@\string#1\endcsname
    \expandafter\@firstoftwo
  \else
    \expandafter\@secondoftwo
  \fi
  {\expandafter\amsmath@bigm\csname fenced@\string#1\endcsname}%
  {\amsmath@bigm#1}%
}

\newcommand{\DeclareFence}[2]{\@namedef{fenced@\string#1}{#2}}
\makeatother

\makeatletter
\let\save@mathaccent\mathaccent
\newcommand*\if@single[3]{%
  \setbox0\hbox{${\mathaccent"0362{#1}}^H$}%
  \setbox2\hbox{${\mathaccent"0362{\kern0pt#1}}^H$}%
  \ifdim\ht0=\ht2 #3\else #2\fi
  }
\newcommand*\rel@kern[1]{\kern#1\dimexpr\macc@kerna}
\newcommand*\widebar[1]{\@ifnextchar^{{\wide@bar{#1}{0}}}{\wide@bar{#1}{1}}}
\newcommand*\wide@bar[2]{\if@single{#1}{\wide@bar@{#1}{#2}{1}}{\wide@bar@{#1}{#2}{2}}}
\newcommand*\wide@bar@[3]{%
  \begingroup
  \def\mathaccent##1##2{%
    \let\mathaccent\save@mathaccent
    \if#32 \let\macc@nucleus\first@char \fi
    \setbox\z@\hbox{$\macc@style{\macc@nucleus}_{}$}%
    \setbox\tw@\hbox{$\macc@style{\macc@nucleus}{}_{}$}%
    \dimen@\wd\tw@
    \advance\dimen@-\wd\z@
    \divide\dimen@ 3
    \@tempdima\wd\tw@
    \advance\@tempdima-\scriptspace
    \divide\@tempdima 10
    \advance\dimen@-\@tempdima
    \ifdim\dimen@>\z@ \dimen@0pt\fi
    \rel@kern{0.6}\kern-\dimen@
    \if#31
      \overline{\rel@kern{-0.6}\kern\dimen@\macc@nucleus\rel@kern{0.4}\kern\dimen@}%
      \advance\dimen@0.4\dimexpr\macc@kerna
      \let\final@kern#2%
      \ifdim\dimen@<\z@ \let\final@kern1\fi
      \if\final@kern1 \kern-\dimen@\fi
    \else
      \overline{\rel@kern{-0.6}\kern\dimen@#1}%
    \fi
  }%
  \macc@depth\@ne
  \let\math@bgroup\@empty \let\math@egroup\macc@set@skewchar
  \mathsurround\z@ \frozen@everymath{\mathgroup\macc@group\relax}%
  \macc@set@skewchar\relax
  \let\mathaccentV\macc@nested@a
  \if#31
    \macc@nested@a\relax111{#1}%
  \else
    \def\gobble@till@marker##1\endmarker{}%
    \futurelet\first@char\gobble@till@marker#1\endmarker
    \ifcat\noexpand\first@char A\else
      \def\first@char{}%
    \fi
    \macc@nested@a\relax111{\first@char}%
  \fi
  \endgroup
}
\makeatother

\newcommand{\vqascoretext}{$\mathsf{VQAScore}$\xspace}
\newcommand{\simplematchtext}{$\mathsf{Simple Match}$\xspace}
\newcommand{\ttmtext}{$\mathsf{TTM}$\xspace}

\newcommand{\textscoremath}{{\mathsf{TextScore}}}
\newcommand{\groupscoremath}{{\mathsf{GroupScore}}}
\newcommand{\groupscoretext}{{$\groupscoremath$}\xspace}
\newcommand{\textscoretext}{{$\textscoremath$}\xspace}

\newcommand{\matchscoremath}{{\mathsf{GroupMatch}}}
\newcommand{\matchscoretext}{$\matchscoremath$\xspace}
\newcommand{\simmath}{{\mathsf{sim}}}
\newcommand{\ftmath}{{\mathsf{ft}}}

\newcommand{\whatsup}{\text{WhatsUp}\xspace}
\newcommand{\clipbaselow}{\text{CLIP-B32}\xspace}

\newcommand{\siglipbase}{\text{SigLIP-B16}\xspace}

\title{Test-Time Matching: Unlocking Compositional Reasoning in Multimodal Models}
\date{}
\author{
\begin{tabular}{ccc}
Yinglun Zhu\textsuperscript{\dag} & Jiancheng Zhang & Fuzhi Tang \\
  {\normalsize\texttt{yzhu@ucr.edu}} & {\normalsize\texttt{jzhan745@ucr.edu}} & {\normalsize\texttt{fuzhit@ucr.edu}} \\
  \multicolumn{3}{c}{{\normalsize University of California, Riverside}}\\[0.3em]
  \multicolumn{3}{c}{{\normalsize Code: \url{https://github.com/yinglunz/test-time-matching}}}
\end{tabular}
}

\begin{document}

\maketitle
\begingroup
\renewcommand\thefootnote{}\footnotetext{\textsuperscript{\dag}Project lead and corresponding author.}
\endgroup

\begin{abstract}
  
Frontier AI models have achieved remarkable progress, yet recent studies suggest they struggle with \emph{compositional reasoning}, often performing at or below random chance on established benchmarks. We revisit this problem and show that widely used evaluation metrics systematically \emph{underestimate} model capability. To correct this artifact, we introduce a \emph{group matching score} that more faithfully evaluates model capability. Moreover, correctness under the new metric can be translated into correctness under existing metrics via a simple overfitting step.
This adjustment enables SigLIP-B16 to surpass all previous results and GPT-4.1 to \emph{yield the first result surpassing estimated human performance on Winoground.}  
Building on this insight, we propose \emph{Test-Time Matching} (\ttmtext), an iterative, self-improving algorithm that further bootstraps model performance without any external supervision. \ttmtext delivers additional, non-trivial improvements: for example, \textbf{\ttmtext enables SigLIP-B16 to surpass GPT-4.1 on MMVP-VLM, establishing a new state of the art}. 
\ttmtext also extends beyond contrastive vision-language models, yielding clear gains on a generative multimodal model across benchmarks.
Importantly, \ttmtext remains broadly effective even on benchmarks without metric-induced effects or group structures, \textbf{achieving relative gains up to 85.7\%} on challenging datasets such as \whatsup. Across 16 dataset variants spanning diverse setups, our experiments demonstrate that \ttmtext consistently improves model performance and advances the frontier of compositional reasoning.

\end{abstract}

\section{Introduction}
\label{sec:intro}

Compositional reasoning provides a stringent test of frontier AI models, assessing their ability to systematically combine primitive elements---such as objects, attributes, and relations---to interpret or reason about novel configurations \citep{lake2017building, bahdanau2018systematic}.  
Recent benchmarks evaluate this capability by organizing examples into small groups of images and captions that differ in subtle yet systematic ways \citep{thrush2022winoground, hsieh2023sugarcrepe, kamath2023s, tong2024eyes, burapacheep2024colorswap}.  
For example, Winoground consists of $2\times2$ groups where both captions contain the same words but in different orders, such that each caption correctly describes only one of the two images.  

Despite the impressive practicality of modern multimodal systems, both contrastive vision-language models (VLMs) and multimodal large language models (MLLMs) have been reported to perform at or below random guessing on these benchmarks \citep{thrush2022winoground, diwan2022winoground, tong2024eyes, burapacheep2024colorswap, li2024exploring}.  
On Winoground, even frontier AI models still fall far short of the estimated human performance of $85.5$ \citep{thrush2022winoground}, with the previous state of the art reaching only $58.75$, achieved through scaffolding and prompt tuning GPT-4V \citep{wu2023role, vaishnav2025cognitive}.  

We revisit this conclusion and show that the widely used evaluation metric \groupscoretext \citep{thrush2022winoground, tong2024eyes, burapacheep2024colorswap} systematically  \emph{underestimates} model capability.  
We introduce a \emph{group matching score} (\matchscoretext) that more faithfully evaluates model capability.
Importantly, correctness under \matchscoretext can be translated into correctness under \groupscoretext by overfitting to the matchings induced by \matchscoretext, an approach we refer to as \simplematchtext (see \cref{sec:metric}).
\simplematchtext alone reveals substantial hidden capability: 
as shown in \cref{fig:main}, SigLIP-B16 improves from $10.25 \rightarrow 67$ on Winoground, $22.96 \rightarrow 81.48$ on MMVP-VLM, and $30.33 \rightarrow 88$ on ColorSwap, surpassing all previous results without access to additional data \citep{wu2023role, vaishnav2025cognitive, zhang2024vipact, burapacheep2024colorswap}.  
GPT-4.1 also improves dramatically, from $69.75 \rightarrow 91.38$ on Winoground, $68.15 \rightarrow 88.52$ on MMVP-VLM, and $91.08 \rightarrow 97.42$ on ColorSwap---\emph{yielding the first result to surpass the estimated human performance of 85.5 on Winoground} \citep{thrush2022winoground}.\footnote{We use GPT-4.1-2025-04-14, the most recent GPT model available to us that supports log-probability outputs, enabling more accurate computation of similarity scores \citep{lin2024evaluating}. At the time of writing (September 2025), GPT-5 did not support log-probability outputs.}

Building on this insight, we introduce \emph{Test-Time Matching} (\ttmtext), an iterative, self-improving algorithm that further bootstraps model performance without any external supervision. 
\ttmtext selects matching-induced pseudo-labels for self-training and progressively relaxes the selection threshold to expand coverage over the test set. 
This yields \emph{additional, non-trivial} gains on top of \simplematchtext: SigLIP-B16 reaches $72.5$ on Winoground, $89.44$ on MMVP-VLM, and $94.25$ on ColorSwap. 
Remarkably, \ttmtext elevates SigLIP-L16 to the level of GPT-4.1 on ColorSwap (\cref{tab:main_2x2}) and \textbf{enables SigLIP-B16 to surpass GPT-4.1 on MMVP-VLM, establishing a new state of the art.} 
See \cref{fig:main} and \cref{tab:main_2x2} for details. 
\ttmtext also extends beyond contrastive vision-language models, 
yielding clear gains on a generative multimodal model across benchmarks (\cref{tab:smolvlm}).
Crucially, \ttmtext is broadly effective even where metric changes cannot help---on $1\times k$ benchmarks such as SugarCrepe \citep{hsieh2023sugarcrepe} and \whatsup \citep{kamath2023s}, where \groupscoretext and \matchscoretext coincide, \ttmtext still delivers substantial test-time improvements, including \textbf{up to 85.7\% relative gains} on challenging datasets such as \whatsup (\cref{fig:exp_1xk}).
\looseness=-1

\begin{figure}[t]
  \centering
  \includegraphics[width=.9\linewidth]{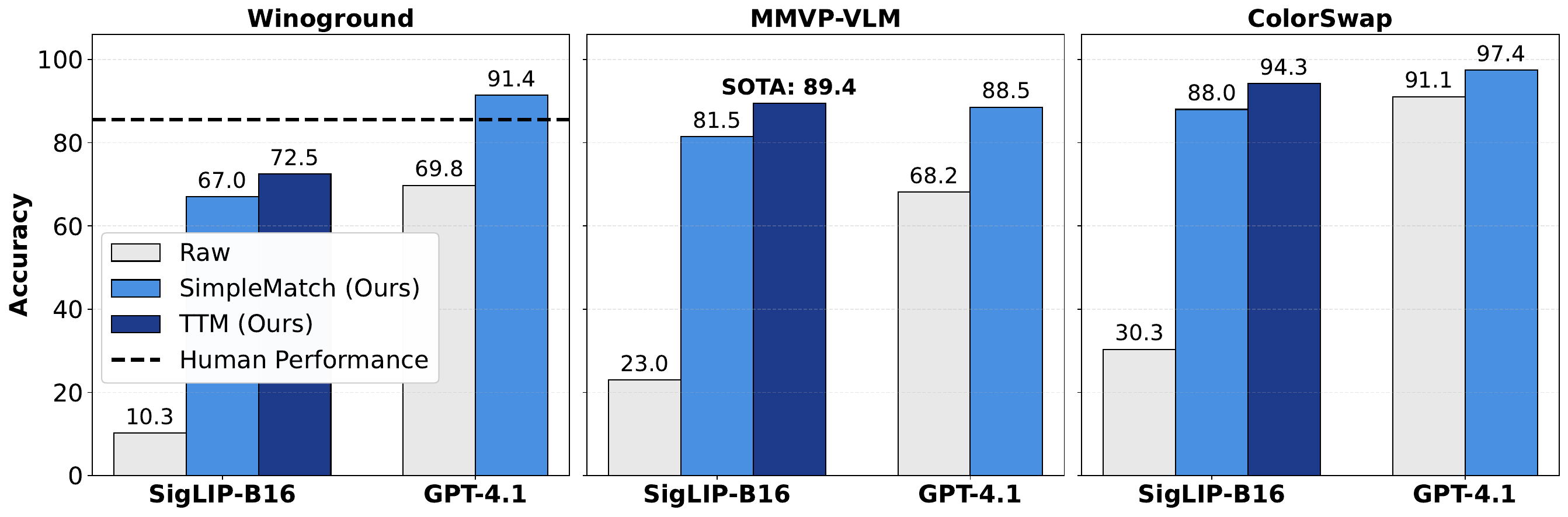}
  \caption{\simplematchtext and \ttmtext substantially improve VLM and MLLM performance on compositional reasoning benchmarks Winoground, MMVP-VLM, and ColorSwap, achieving new performance records. We highlight: (1) \simplematchtext enables GPT-4.1 to surpass human performance on Winoground (\emph{left}), and (2) \ttmtext enables SigLIP-B16 to surpass GPT-4.1 on MMVP-VLM, establishing a new state of the art (\emph{middle}).}
  \looseness=-1
  \label{fig:main}
  \vspace{-8 pt}
\end{figure}

Finally, we extend \ttmtext beyond group-structured datasets by formulating a single global matching across all images and captions.  
Even a one-shot global matching outperforms raw \groupscoretext, and applying the global variant of \ttmtext yields further improvements, demonstrating that the test-time matching principle generalizes robustly beyond benchmarks with group structures.

  \vspace{-5 pt}
\paragraph{Contributions.} We summarize our main contributions below:
\begin{enumerate}
  [leftmargin=10pt, itemindent=*]
  \item \textbf{Correcting evaluation metrics.} 
  We show that the widely used evaluation metric \groupscoretext systematically \emph{underestimates} model capability, and introduce \matchscoretext as a faithful measure.
  We further show that a simple overfitting step translates correctness under \matchscoretext into correctness under \groupscoretext, enabling GPT-4.1 to achieve the first Winoground result surpassing estimated human performance.

  \item \textbf{Test-time matching for self-improvements.} We propose \emph{Test-Time Matching} (\ttmtext), an iterative, self-improving algorithm 
that selects matching-induced pseudo-labels for self-training and progressively relaxes the selection threshold to expand coverage.
\ttmtext delivers additional, non-trivial gains on top of \simplematchtext, enabling SigLIP-B16 to surpass GPT-4.1 on MMVP-VLM and establishing a new state of the art.

  \item \textbf{Broad applicability of \ttmtext.} 
  We conduct extensive experiments across 16 dataset variants spanning $2\times2$, $1\times k$, and non-grouped settings, demonstrating that \ttmtext consistently improves model performance across diverse scenarios, including those without metric-induced effects or predefined group structures.
  \looseness=-1
\end{enumerate}

\paragraph{Paper organization.}
In \cref{sec:preliminary}, we review group-structured evaluation for compositional reasoning.
In \cref{sec:methods}, we revisit evaluation metrics, introduce a new group matching score (\matchscoretext), present our test-time matching (\ttmtext) algorithm, and extend it to global (non-grouped) settings.
In \cref{sec:experiments}, we report results on benchmarks with $2\times2$ groups, $1\times k$ groups, and non-grouped structures, together with ablations and analysis.
We discuss related work in \cref{sec:related} and conclude in \cref{sec:discussion}.
Formal proofs, additional experimental details, and extended results are provided in the Appendix.

\section{Preliminaries}
\label{sec:preliminary}

We study compositional reasoning in multimodal models.  
Benchmarks for this task are typically organized into \emph{groups} of images and captions, often of shape $k \times k$ or $1 \times k$.  
Within each group, the images and captions differ in subtle yet systematic ways.  
For example, the widely used Winoground dataset consists of groups with two images and two captions, where both captions contain the same set of words but in different orders, such that each caption correctly describes only one of the two images \citep{thrush2022winoground}.  

To succeed on these benchmarks, a model must correctly align images and captions within each group.
Let $s_{ij} \coloneqq s(I_i, C_j)$ denote the similarity score between image $I_i$ and caption $C_j$.  
For contrastive vision-language models such as CLIP \citep{radford2021learning} and SigLIP \citep{zhai2023sigmoid}, $s_{ij}$ is typically computed as the inner product of image and text embeddings.  
For multimodal large language models, similarity can instead be estimated using metrics such as \vqascoretext \citep{lin2024evaluating}.  
We collect all scores into a similarity matrix $s$, which shares the same shape as the group.  

\paragraph{The \groupscoretext metric for $k \times k$ groups.}  
Consider a group of $k$ images and $k$ captions with ground-truth pairings 
$\crl{\prn{I_i, C_i}}_{i=1}^k$ hidden from the learner. 
The most widely used evaluation metric is the \groupscoretext \citep{thrush2022winoground, tong2024eyes, burapacheep2024colorswap}.  
The \groupscoretext equals $1$ if the model's similarity scores admit a bijection such that (i) each image is assigned to its correct caption and (ii) each caption is assigned to its correct image; otherwise it equals $0$.  
Mathematically, we have
\begin{align}
\label{eq:group_score}
\groupscoremath(s) \ldef
\begin{cases}
1 & 
\forall i:\; s_{ii} > \max_{j \neq i} s_{ij}
\quad \text{and} \quad
s_{ii} > \max_{j \neq i} s_{ji}, \\[6pt]
0 & \text{otherwise}.
\end{cases}
\end{align}

\paragraph{Evaluation metrics for $1 \times k$ groups.}  
Without loss of generality, we assume each group consists of $1$ image and $k$ captions \citep{kamath2023s, hsieh2023sugarcrepe}.
In this case, the \groupscoretext reduces to the \textscoretext, which equals $1$ if the model selects the correct caption and 0 otherwise.

\paragraph{Scope and extensions.}
In this paper, we primarily focus on $k \times k$ and $1 \times k$ groups as they are the most common configurations in compositional reasoning benchmarks. We defer discussion of general rectangular groups of shape $m \times k$ to \cref{app:rectangular}.

\section{Methods}
\label{sec:methods}

In \cref{sec:metric}, we show that the standard evaluation metric can systematically \emph{undercount} model capability on group-structured benchmarks, and we introduce a group matching score (\matchscoretext) that corrects this artifact.
Building on this, we develop an iterative, self-improving \emph{Test-Time Matching} (\ttmtext) algorithm that bootstraps model performance without external supervision (\cref{sec:ttm}). 
Finally, we extend \ttmtext beyond group-structured datasets to a global matching formulation applicable to general settings (\cref{sec:global_matching}).

\subsection{Revisiting evaluation metrics: from random guessing to group matching}
  \label{sec:metric}

  Most compositional reasoning benchmarks use the \groupscoretext metric described in \cref{sec:preliminary}. 
Despite the broad practical success of frontier AI models, reported results on established benchmarks---particularly those with $k \times k$ groups---are often \emph{at or below random guessing} \citep{thrush2022winoground, diwan2022winoground, tong2024eyes, burapacheep2024colorswap, li2024exploring}.\footnote{These benchmarks are widely adopted; for example, as of October 2025, Winoground \citep{thrush2022winoground} has over 500 citations and MMVP-VLM \citep{tong2024eyes} has nearly 500 citations.}

\paragraph{Revisiting evaluation metrics.}
Such counter-intuitive results motivate us to re-examine evaluation metrics for $k \times k$ groups.
To calibrate their behavior, we analyze a \emph{random guessing model} under each metric.
Consider a group of $k$ images $\{I_i\}_{i=1}^k$ and $k$ captions $\{C_i\}_{i=1}^k$, with ground-truth pairings $\{(I_i, C_i)\}_{i=1}^k$ hidden from the learner \citep{thrush2022winoground, tong2024eyes, burapacheep2024colorswap}.  
For each pair $(I_i, C_j)$, the random guessing model assigns a similarity score $\simmath(I_i, C_j) \sim \unif([0,1])$,
producing a similarity matrix $s \in \bbR^{k \times k}$ with entries $s_{ij} \ldef \simmath(I_i, C_j)$.  

Under the widely used \groupscoretext metric, achieving a score of $1$ requires the similarity matrix $s$ to satisfy $2k^2 - 2k$ constraints (see \cref{eq:group_score}).  
Equivalently, each diagonal entry $s_{ii}$ must be the largest element in both its row and column---a highly restrictive condition.  
The probability of achieving a group score of 1 under random guessing is given below (see \cref{app:proofs_group} for proofs).  

\begin{restatable}{proposition}{propGroupScoreProb}
\label{prop:group_score_prob}
  For random similarity scores $s \in \bbR^{k \times k}$, 
  $\bbP(\groupscoremath(s) = 1) = \frac{(k-1)!}{(2k-1)!}$.
\end{restatable}

\paragraph{A new group matching score.}
We introduce a new evaluation metric that evaluates the \emph{best overall matching} rather than isolated pairwise comparisons.  
We consider \emph{bijective matchings} (one-to-one and onto) from images to captions.
Let $\pi$ denote such a matching, where $\pi(i)$ is the caption assigned to image $i$.  
We define the \matchscoretext as 
\looseness=-1
\begin{align*}
  \matchscoremath(s) \ldef 
  \begin{cases}
    1 & \text{if } \sum_{i=1}^k s_{i,\pi^\star(i)} > \sum_{i=1}^k s_{i,\pi(i)}, \quad \forall\; \pi \neq \pi^\star, \\[3pt]
    0 & \text{otherwise},
  \end{cases}
\end{align*}
where $\pi^\star: i \mapsto i$ denotes the ground-truth matching.  
Intuitively, the \matchscoretext equals $1$ if the \emph{total similarity} of the ground-truth matching exceeds that of all other possible matchings.  
For $k=2$, this reduces to the simple condition $s_{11} + s_{22} > s_{12} + s_{21}$.  
Since there are $k!$ distinct matchings (permutations) and, under random guessing, each is equally likely to maximize the total score, we obtain the following result.  

\begin{restatable}{proposition}{propMatchScoreProb}
\label{prop:match_score_prob}
  For random similarity scores $s \in \bbR^{k \times k}$, 
  $\bbP(\matchscoremath(s) = 1) = \frac{1}{k!}$.
\end{restatable}

\begin{remark}
  The \matchscoretext naturally extends to general rectangular groups of shape $m \times k$ (with $m < k$) by considering all injective matchings (one-to-one).
In these cases, it also improves over the \groupscoretext, increasing the expected random guessing score from $1/k^m$ to $(k-m)!/k!$ (see \cref{app:rectangular} for details).
In the special case of $1 \times k$ groups, the \matchscoretext and the \groupscoretext coincide.
\end{remark}

\paragraph{\simplematchtext: translating \matchscoretext to \groupscoretext.}  
Two key observations emerge:  
\begin{itemize}
  [leftmargin=10pt, itemindent=*]
  \item 
  $\bbP(\matchscoremath(s) = 1) > \bbP(\groupscoremath(s) = 1)$ for all integers $k > 1$.  
  \item If the correct matching $\pi^\star$ is selected, overfitting to $\pi^\star$ at test time guarantees a \groupscoretext of $1$.  
  In other words, correctness under \matchscoretext can be translated into correctness under \groupscoretext.
\end{itemize}

Taken together, these observations imply that \textbf{\groupscoretext systematically underestimates model capability}: it imposes unnecessary constraints and can miscount correct matchings as wrong. 
In fact, one can easily improve model performance under \groupscoretext by 
(i) selecting the most likely matching under \matchscoretext and (ii) 
overfitting to that matching at test time to transfer gains.\footnote{Since overfitting to matchings induced by \matchscoretext achieves the same level of performance under \groupscoretext, throughout the paper, we report raw model performance under \groupscoretext and our algorithms' performance under \matchscoretext. The latter can always be converted to equivalent \groupscoretext performance with an additional overfitting step.}  
We refer to this approach as \simplematchtext with \matchscoretext.
In the commonly studied case with $k=2$, the expected group score of a random guessing model increases from $1/6$ to $1/2$.

\paragraph{Empirical validation.}
We evaluate \simplematchtext on SigLIP \citep{zhai2023sigmoid} and GPT-4.1 across three established compositional reasoning benchmarks with $k \times k$ group structures: Winoground \citep{thrush2022winoground}, MMVP-VLM \citep{tong2024eyes}, and Colorswap \citep{burapacheep2024colorswap}.  
Results are presented in \cref{fig:main}.
\simplematchtext reveals substantial hidden capability: 
SigLIP-B16 improves from $10.25 \rightarrow 67$ on Winoground, $22.96 \rightarrow 81.48$ on MMVP-VLM, and $30.33 \rightarrow 88$ on ColorSwap, surpassing all previous results without access to additional data \citep{wu2023role, vaishnav2025cognitive, zhang2024vipact, burapacheep2024colorswap}.  
GPT-4.1 also improves dramatically, from $69.75 \rightarrow 91.38$ on Winoground, $68.15 \rightarrow 88.52$ on MMVP-VLM, and $91.08 \rightarrow 97.42$ on ColorSwap---\emph{yielding the first result to surpass the estimated human performance of 85.5 on Winoground} \citep{thrush2022winoground}.

\subsection{Test-Time Matching: iterative bootstrapping of model performance}
\label{sec:ttm}

The \matchscoretext metric introduced in \cref{sec:metric} reveals substantial model capability masked by previous metrics. 
To push performance further, we introduce a test-time matching algorithm that iteratively bootstraps model performance, yielding new state-of-the-art results. 
Our method applies to groups of general shapes: we consider bijective matchings for square groups and injective matchings for rectangular groups.

\paragraph{High-level idea.}
Our test-time matching algorithm (\cref{alg:ttm}) proceeds iteratively for $T$ iterations.
At each round $t \in [T]$, the current model $f_{t-1}$ induces candidate matchings for all groups, which serve as pseudo-labels. 
The algorithm then retains only those matchings it is most confident about, and finetunes on them to obtain the next model $f_t$. 
By repeating this process, the model progressively self-improves directly at test time, without any external supervision.

\begin{algorithm}[tbp]
\caption{Test-Time Matching (\ttmtext)}
\label{alg:ttm}
\renewcommand{\algorithmicrequire}{\textbf{Input:}}
\renewcommand{\algorithmicensure}{\textbf{Output:}}
\newcommand{\algorithmicbreak}{\textbf{break}}
\newcommand{\BREAK}{\STATE \algorithmicbreak}
\begin{algorithmic}[1]
\REQUIRE Pretrained $f_0$; test set of groups $\mathcal{D}=\{G_i\}_{i=1}^n$; number of iterations $T$; thresholds $\{\tau_t\}_{t=1}^T$.
\FOR{iteration $t=1$ to $T$}
  \STATE Initialize pseudo-labeled set $\mathcal{S}_t \gets \emptyset$.
  \FOR{each group $G_i \in \mathcal{D}$}
    \STATE Induce matching $\pi_{f_{t-1}}(G_i) \gets \argmax_{\pi} s(\pi; G_i, f_{t-1})$.
    \STATE Compute margin $\Delta(G_i; f_{t-1})$ as 
  \vspace{-5pt}
    $$\Delta(G_i; f_{t-1}) \gets s(\pi_{f_{t-1}}(G_i); G_i, f_{t-1}) - \max_{\pi \neq \pi_{f_{t-1}}(G_i)} s(\pi; G_i, f_{t-1}).$$
  \vspace{-15pt}
    \IF{$\Delta(G_i; f_{t-1}) \ge \tau_t$}
      \STATE $\mathcal{S}_t \gets \mathcal{S}_t \cup \{(G_i, \pi_{f_{t-1}}(G_i))\}$.
    \ENDIF
  \ENDFOR
  \STATE Finetune model on $\mathcal{S}_t$ to obtain $f_t$. 
  \algcommentlight{Self-improving with no external supervision.}
\ENDFOR
\ENSURE Test-time adapted model $f_T$.
\end{algorithmic}
\end{algorithm}

The core of \cref{alg:ttm} lies in two design choices: 
(1) how pseudo-labels are induced within each group, and 
(2) how the confidence thresholds are scheduled across iterations. 
We discuss both below.

\paragraph{Group matching and pseudo-labeling.}
For a group $G$ and model $f_{t-1}$, we define the induced matching
$\pi_{f_{t-1}}(G) \ldef \argmax_{\pi} s(\pi; G, f_{t-1})$,
where
$s(\pi; G, f_{t-1}) \ldef \sum_{u} s_{u,\pi(u)}(G; f_{t-1})$
denotes the total similarity of matching $\pi$ on $G$ under $f_{t-1}$. 
For example, in a $2\times 2$ group, $\pi_{f_{t-1}}(G)=(1\!\mapsto\!1,\,2\!\mapsto\!2)$ if $s_{11}+s_{22} > s_{12}+s_{21}$, and $(1\!\mapsto\!2,\,2\!\mapsto\!1)$ otherwise. 
For a $1\times k$ group, the induced matching is $(1\!\mapsto\!\arg\max_{j\in[k]} s_{1j})$. 
We convert $\pi_{f_{t-1}}(G)$ into a pseudo-label $({G, \pi_{f_{t-1}}(G)})$ and add it to the training set $\cS_t$ only when its \emph{margin}
\begin{align*}
\Delta(G; f_{t-1}) \ldef s(\pi_{f_{t-1}}(G); G, f_{t-1}) - \max_{\pi \neq \pi_{f_{t-1}}(G)} s(\pi; G, f_{t-1})
\end{align*}
is greater than or equal to a threshold $\tau_t$. 
By controlling the threshold, we ensure that the model retains pseudo-labels it is sufficiently confident about.

\begin{figure}[t]
  \centering
  \begin{minipage}[t]{0.33\linewidth}
    \vspace{0pt}\centering
    \includegraphics[width=\linewidth]{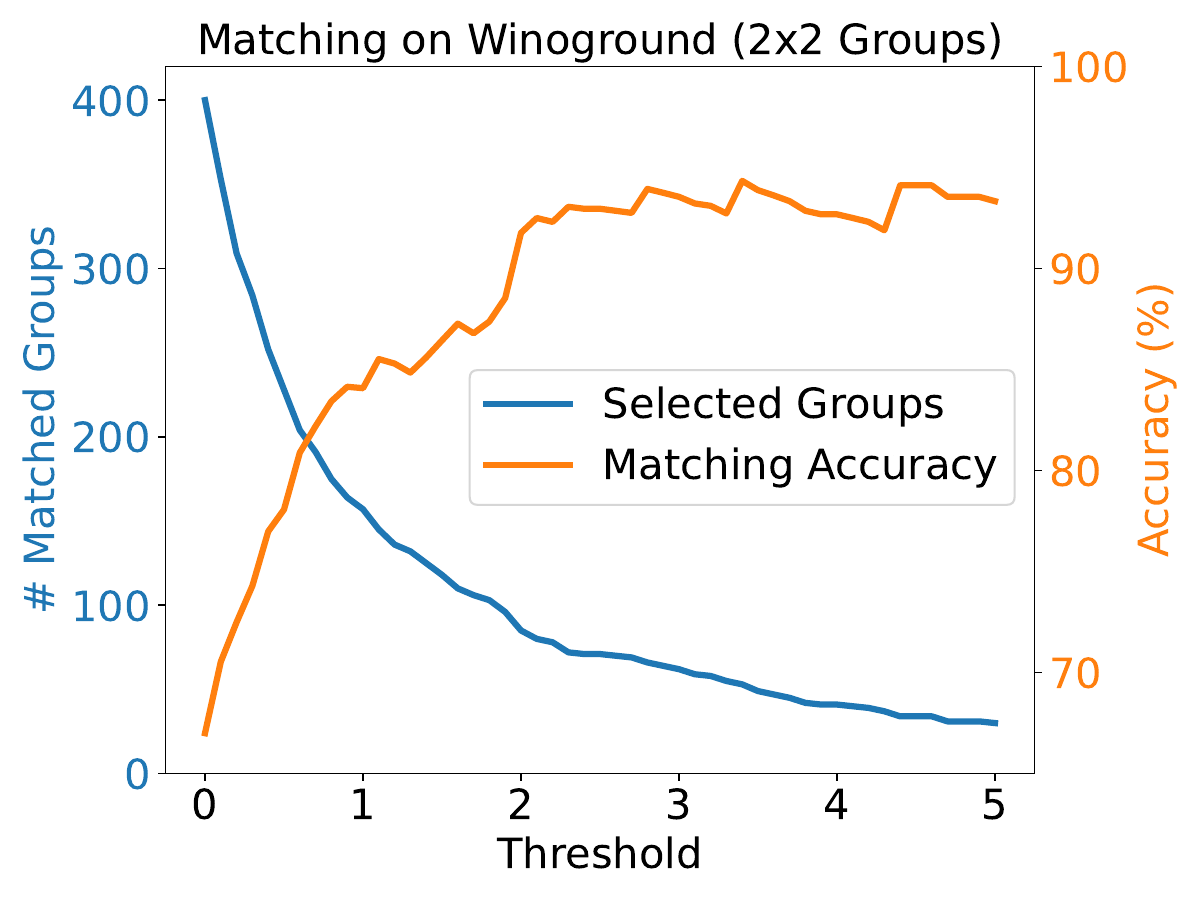}
  \end{minipage}\hfill
  \begin{minipage}[t]{0.33\linewidth}
    \vspace{0pt}\centering
    \includegraphics[width=\linewidth]{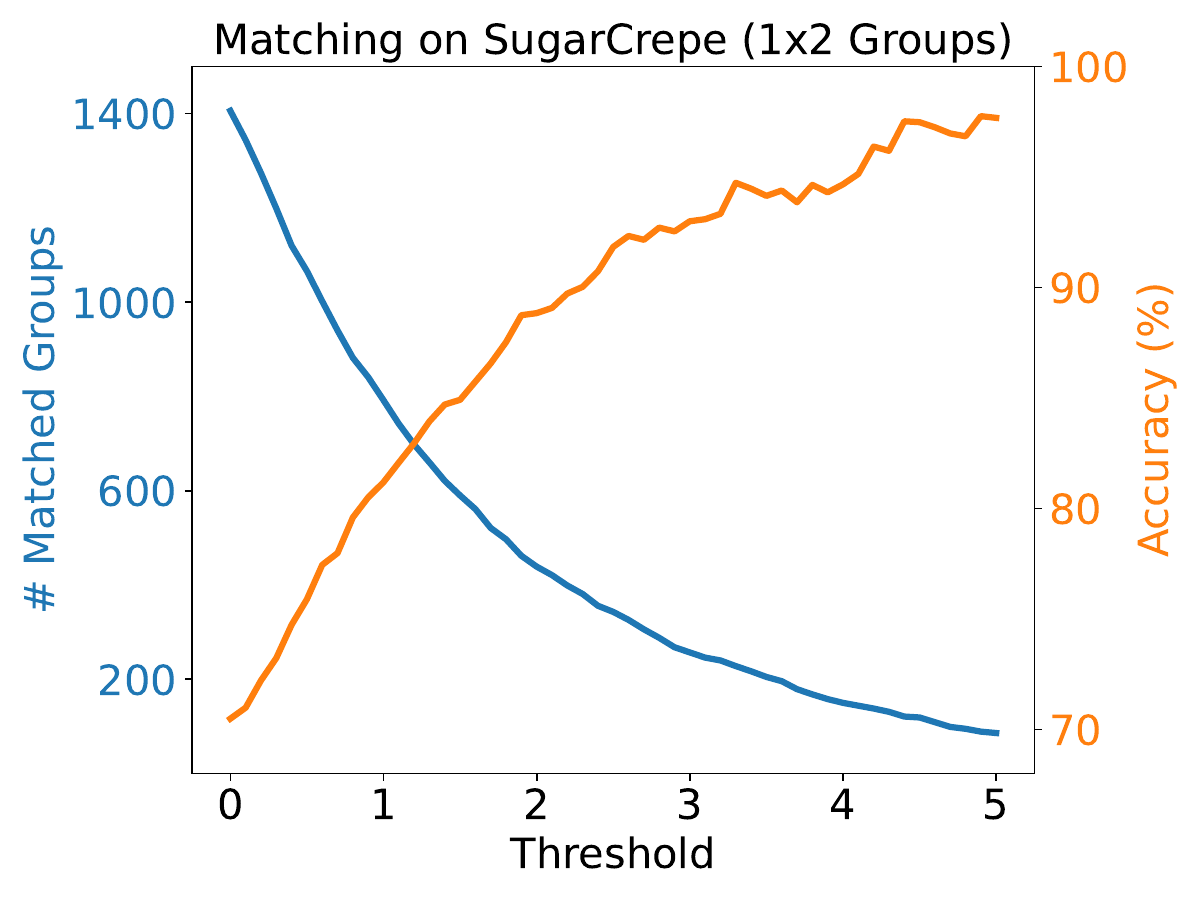}
  \end{minipage}\hfill
  \begin{minipage}[t]{0.31\linewidth}
    \vspace{0pt}\centering
    \includegraphics[width=\linewidth]{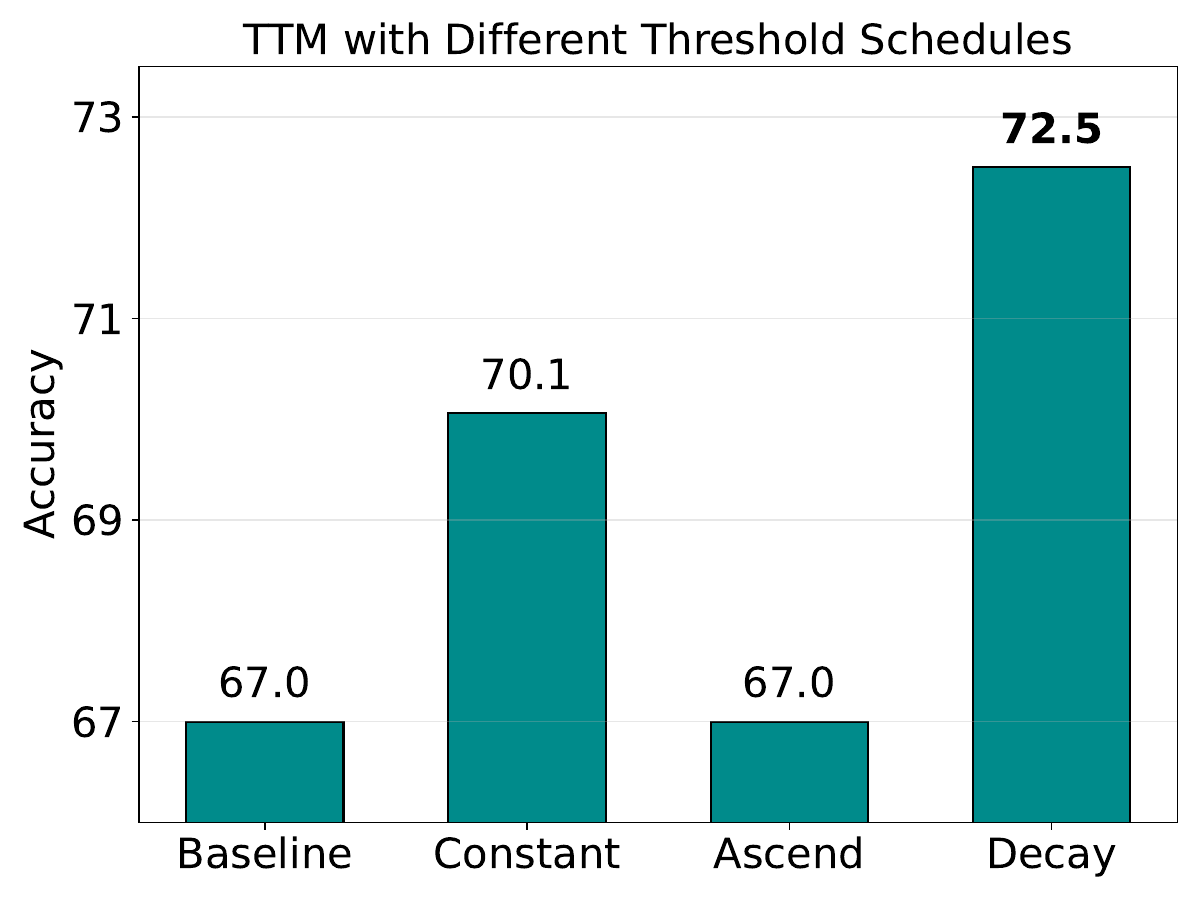}
  \end{minipage}
  \caption{
  \emph{Left and middle:} Matching results across different thresholds on Winoground and SugarCrepe (the Replace Relation subset) with SigLIP-B16.  
\emph{Right:} Performance of \ttmtext under different threshold schedules on Winoground with SigLIP-B16.  
  \emph{Baseline} denotes model performance without \ttmtext (under \matchscoretext).  
\emph{Constant} applies \ttmtext with a fixed threshold $\tau_t = 2.0$.  
\emph{Ascend} applies \ttmtext with a linearly increasing schedule from $\tau_1 = 0$ to $\tau_T = 2.0$, but yields no gains as the model quickly overfits to all pseudo-labels in the first iteration.  
\emph{Decay} applies \ttmtext with a linearly decreasing schedule from $\tau_1 = 2.0$ to $\tau_T = 0$, yielding the best performance.
  \looseness=-1
}
  \label{fig:match_threshold}
\end{figure}

\paragraph{A decaying selection threshold schedule.}
Any pseudo-label-based method (including \ttmtext) faces two types of pseudo-labeling errors:
(i) \emph{false positives}---incorrect pseudo-labels that are mistakenly included in the training set, and
(ii) \emph{false negatives}---correct predictions that are excluded due to overly strict selection criteria.
In matching-based pseudo-labeling, a lower threshold $\tau$ increases the number of selected pseudo-labels but typically reduces precision (more false positives),
whereas a higher threshold yields cleaner but fewer pseudo-labels (more false negatives).
This trade-off is illustrated in the left and middle plots of \cref{fig:match_threshold}, which report the number of matched groups (blue) and the accuracy among matched groups (orange) as a function of $\tau$.
\looseness=-1

\ttmtext employs a \emph{decaying} threshold schedule ($\tau_{t+1} < \tau_t$) to balance these errors over iterations:
it begins with a high threshold to ensure high-precision pseudo-labels (few false positives), and then gradually lowers the threshold to increase coverage and reduce false negatives as the model improves.
In contrast, an \emph{ascending} schedule can admit substantial false positives in early rounds, which derails learning and limits subsequent improvement; a \emph{constant} schedule avoids early false positives by keeping the threshold high, but induces many false negatives in later iterations, leading to limited gains and early plateauing.
The right plot of \cref{fig:match_threshold} supports this intuition: the decaying schedule consistently outperforms the alternatives.

In practice, we find it effective to set the initial threshold $\tau_1$ such that roughly 15\%--30\% of the groups are matched, and the final threshold $\tau_T$ such that more than 90\% of the test set is covered. 
Both cosine and linear decay schedules perform well. 
Further analyses and ablations are provided in \cref{sec:analyses}.
\looseness=-1

\paragraph{Runtime analysis.}
The runtime of \ttmtext scales as $O(T \cdot C_\ftmath)$, where $T$ is the number of iterations and $C_\ftmath$ denotes the per-iteration model finetuning cost.
Empirically, \ttmtext demonstrates strong improvements even with a small number of iterations (e.g., $T=3$ or $10$).
Thus, the runtime of \ttmtext is comparable to standard test-time training methods in the literature \citep{sun2020test, akyurek2025the}.
See \cref{app:runtime} for a detailed analysis and discussion.

\subsubsection{Test-Time Matching without group structures}
\label{sec:global_matching}

While \cref{alg:ttm} is designed for datasets organized into local groups, the same principle extends naturally to settings without any predefined group structure. 
In this case, we treat the entire dataset as a single global matching problem between all images and all captions.

Let $\cS_I$ denote the set of images and $\cS_C$ the set of captions. 
We assume $\abs{\cS_I} \le \abs{\cS_C}$ and each image has a unique corresponding caption (one-to-one assignment). 
Let $s \in \bbR^{\abs{\cS_I} \times \abs{\cS_C}}$ be the similarity matrix produced by a model $f$. 
We consider all injective matchings $\pi: \cS_I \rightarrow \cS_C$ from images to captions.
The model-induced global matching is then defined as
\begin{align}
  \label{eq:global_matching}
  \pi_f \ldef \argmax_{\pi:\,\cS_I \to \cS_C} \;
  \sum_{i \in \cS_I} s_{i,\pi(i)},
\end{align}
which maximizes the total similarity over image-caption pairs. 
\cref{eq:global_matching} corresponds to the classical \emph{assignment problem}, which can be efficiently solved by strongly-polynomial time algorithms such as the Hungarian algorithm \citep{kuhn1955hungarian}.

Analogous to \cref{alg:ttm}, we adopt an iterative schedule with pseudo-labeling. 
At iteration $t$, let $\pi_{f_{t-1}}$ be the global matching induced by model $f_{t-1}$. 
Because the entire dataset is treated as a single group, group-level margin thresholding loses granularity: the model would either accept all matches or none.
To address this, we apply thresholding at the level of individual pairs. 
Specifically, the pseudo-label set at iteration $t$ is
\begin{align*}
  \cS_t \ldef \crl[\bigr]{ (i, \pi_{f_{t-1}}(i)) :
    s_{i, \pi_{f_{t-1}}(i)}  \geq \tau_t
  },
\end{align*}
where $\tau_t$ is the threshold at iteration $t$. 
The threshold can be set either as an absolute value or relative to the distribution of similarity scores (i.e., the $p$-th percentile). 
Following the same principle as in \cref{alg:ttm}, we begin with a relatively high threshold to ensure high-precision pseudo-labels and gradually decay it over iterations to expand coverage and bootstrap performance over the test set.

\section{Experiments}
\label{sec:experiments}

We describe experimental setups in \cref{sec:exp_setups}, present main results in \cref{sec:exp_2x2,sec:exp_without_metric_boost,sec:exp_global_matching}, and provide analyses and ablations in \cref{sec:analyses}. Additional experimental details and results are deferred to \cref{app:exp}.

\subsection{Experimental setups}
\label{sec:exp_setups}

\paragraph{Datasets.}
We evaluate on five challenging compositional reasoning benchmarks: 
Winoground \citep{thrush2022winoground}, MMVP-VLM \citep{tong2024eyes}, 
Colorswap \citep{burapacheep2024colorswap}, SugarCrepe \citep{hsieh2023sugarcrepe}, 
and \whatsup \citep{kamath2023s}. 
Winoground, MMVP-VLM, and Colorswap consist of $2\times 2$ groups; we also construct their non-grouped variants by discarding group structures (\cref{sec:global_matching}). 
SugarCrepe consists of $1\times 2$ groups and \whatsup consists of $1\times 4$ groups; we evaluate on 4 different subsets of SugarCrepe and all 2 subsets of \whatsup. 
Following \citet{li2024exploring}, we further convert \whatsup into 4 different variants with $2\times 2$ groups. 
In total, our evaluation spans 16 dataset variations covering diverse structures and evaluation settings.

\paragraph{Models.}
We test both contrastive vision-language models and multimodal large language models.  
For contrastive models, we use SigLIP \citep{zhai2023sigmoid} and CLIP \citep{radford2021learning} at multiple scales, including SigLIP-B16, SigLIP-L16, CLIP-B16, and CLIP-B32.  
For multimodal large language models, we use GPT-4.1, where image-text similarity is computed based on \vqascoretext \citep{lin2024evaluating}. 

\paragraph{Evaluation metrics.}
For GPT-4.1, we report raw \groupscoretext and \matchscoretext-induced performance via \simplematchtext (\cref{sec:metric}).  
For CLIPs and SigLIPs, we additionally include results with \ttmtext (\cref{alg:ttm}).  
Specifically: on $2 \times 2$ datasets we report (i) raw \groupscoretext, (ii) \matchscoretext-induced performance, and (iii) \ttmtext-boosted performance; on $1 \times k$ datasets we report (i) raw \groupscoretext and (ii) \ttmtext-boosted performance, since \groupscoretext and \matchscoretext coincide in this case; and on datasets without group structures we report (i) raw \groupscoretext (with known groups), (ii) global assignment accuracy under \cref{eq:global_matching}, and (iii) \ttmtext-boosted performance via the global variant introduced in \cref{sec:global_matching}.  
In all cases, we highlight performance gains from \ttmtext---over \matchscoretext for $2 \times 2$ datasets, over \groupscoretext for $1 \times k$ datasets, and over global assignment accuracy under \cref{eq:global_matching} for datasets without group structures.
All results are averaged over four random runs, with standard deviations reported.

\subsection{\ttmtext achieves new SOTAs}
\label{sec:exp_2x2}

\begin{table}[t]
\centering
\caption{Performance on Winoground, MMVP-VLM, and ColorSwap. 
Raw model performance is reported under \groupscoretext. 
\simplematchtext corresponds to the performance under \matchscoretext (\cref{sec:metric}), 
and \ttmtext corresponds to the performance of \cref{alg:ttm}. 
We report absolute gains ($\Delta$), relative gains, and relative error reductions of \ttmtext over \simplematchtext.
Cells highlighted in \protect\legendpatch{lightblue} indicate results obtained with \ttmtext, 
while cells in \protect\legendpatch{sotacolor} denote the \textbf{SOTA} performance for each dataset.}
  \label{tab:main_2x2}
\begin{tabular}{l c c c c c}
\toprule
{Dataset / Model} & {Raw} & {\simplematchtext} & {\ttmtext} & $\Delta$ & {Error Red.} \\
\midrule
\multicolumn{6}{l}{\textbf{Winoground}} \\
\quad GPT-4.1     & 69.75\,\scriptsize$\pm$\,0.56 & \cellcolor{sotacolor}91.38\,\scriptsize$\pm$\,0.80 & -- & -- & -- \\
\quad CLIP-B16    & 7.25 & 60.00 & \cellcolor{lightblue}\textbf{65.44\,\scriptsize$\pm$\,1.10} 
  & \cellcolor{lightblue}{+}\,\textbf{5.4} \, (\textbf{9.1\%} $\uparrow$) 
  & \cellcolor{lightblue}\textbf{13.6\%} $\downarrow$ \\
\quad SigLIP-B16  & 10.25  & 67.00 & \cellcolor{lightblue}\textbf{72.50\,\scriptsize$\pm$\,0.64} 
  & \cellcolor{lightblue}{+}\,\textbf{5.5} \, (\textbf{8.2\%} $\uparrow$) 
  & \cellcolor{lightblue}\textbf{16.7\%} $\downarrow$ \\
\quad SigLIP-L16  & 13.00 & 69.50 & \cellcolor{lightblue}\textbf{72.75\,\scriptsize$\pm$\,0.64} 
  & \cellcolor{lightblue}{+}\,\textbf{3.3} \, (\textbf{4.7\%} $\uparrow$) 
  & \cellcolor{lightblue}\textbf{10.7\%} $\downarrow$ \\
  \midrule
  \multicolumn{6}{l}{\textbf{MMVP-VLM}} \\
\quad GPT-4.1     & 68.15\,\scriptsize$\pm$\,0.00 & 88.52\,\scriptsize$\pm$\,0.83 & -- & -- & -- \\
\quad CLIP-B16    & 5.19  & 72.59 & \cellcolor{lightblue}\textbf{80.19\,\scriptsize$\pm$\,0.81} 
  & \cellcolor{lightblue}{+}\,\textbf{7.6} \, (\textbf{10.5\%} $\uparrow$) 
  & \cellcolor{lightblue}\textbf{27.7\%} $\downarrow$ \\
\quad SigLIP-B16  & 22.96 & 81.48 & \cellcolor{sotacolor}\textbf{89.44\,\scriptsize$\pm$\,0.96} 
  & \cellcolor{lightblue}{+}\,\textbf{8.0} \, (\textbf{9.8\%} $\uparrow$) 
  & \cellcolor{lightblue}\textbf{43.0\%} $\downarrow$ \\
\midrule
\multicolumn{6}{l}{\textbf{ColorSwap}} \\
  \quad GPT-4.1     & 91.08\,\scriptsize$\pm$\,0.28 & \cellcolor{sotacolor}{97.42\,\scriptsize$\pm$\,0.14} & -- & -- & -- \\
  
\quad CLIP-B16    & 12.00 & 77.67 & \cellcolor{lightblue}\textbf{85.75\,\scriptsize$\pm$\,0.64} 
  & \cellcolor{lightblue}{+}\,\textbf{8.1} \, (\textbf{10.4\%} $\uparrow$) 
  & \cellcolor{lightblue}\textbf{36.2\%} $\downarrow$ \\
\quad SigLIP-B16  & 30.33 & 88.00 & \cellcolor{lightblue}\textbf{94.25\,\scriptsize$\pm$\,0.43} 
  & \cellcolor{lightblue}{+}\,\textbf{6.3} \, (\textbf{7.1\%} $\uparrow$) 
  & \cellcolor{lightblue}\textbf{52.1\%} $\downarrow$ \\
\quad SigLIP-L16  & 37.00 & 91.33 & \cellcolor{lightblue}\textbf{96.08\,\scriptsize$\pm$\,0.43} 
  & \cellcolor{lightblue}{+}\,\textbf{4.8} \, (\textbf{5.2\%} $\uparrow$) 
  & \cellcolor{lightblue}\textbf{54.8\%} $\downarrow$ \\
\bottomrule
\end{tabular}
\end{table}

We evaluate on three established compositional reasoning benchmarks---Winoground, MMVP-VLM, and ColorSwap---all consisting of $2 \times 2$ groups and considered challenging for frontier AI models. 
Previous state-of-the-art results include $58.75$ on Winoground (GPT-4V with prompt tuning \citep{wu2023role, vaishnav2025cognitive}), $70.7$ on MMVP (via a GPT-4o multi-agent system with tool use \citep{zhang2024vipact}),\footnote{This result is on MMVP, a variant of MMVP-VLM formulated as binary-choice question answering. In this paper, we focus on MMVP-VLM, which is better suited for contrastive models. Prior work has shown that model performance on the two variants is positively correlated \citep{tong2024eyes, li2024exploring}.}
and $87.33$ on ColorSwap without training-set access ($95.33$ with finetuning on the training set \citep{burapacheep2024colorswap}).
\looseness=-1

\paragraph{Simple matching reveals hidden capabilities.}  
Applying \simplematchtext (\cref{sec:metric}) to CLIP, SigLIP, and GPT-4.1 already yields striking improvements (\cref{tab:main_2x2}).  
\simplematchtext enables SigLIP-B16 to surpass all prior state-of-the-art results without access to additional data,   
and enables GPT-4.1 to set new records across all three benchmarks.  
Notably, GPT-4.1 improves from 69.75 to 91.38 on Winoground, \emph{yielding the first result to surpass the estimated human performance of 85.5} \citep{thrush2022winoground}.  
These findings confirm that the \matchscoretext metric can reveal substantial hidden compositional reasoning capabilities.  
\looseness=-1

\paragraph{Test-time matching further boosts performance.}  
We next apply \ttmtext (\cref{alg:ttm}) to CLIP and SigLIP, enabling additional performance gains without external supervision.  
As shown in \cref{tab:main_2x2}, \emph{\ttmtext consistently improves over \simplematchtext across datasets and model scales, with relative gains up to 10.5\% and relative error reduction up to 54.8\%}.\footnote{While the absolute boosts may appear modest compared to \matchscoretext-induced gains, they are \emph{highly significant}: for comparison, scaffolding GPT-4V yields only a 1.25-point gain on the Winoground dataset, improving performance from 50.75 \citep{zhang2024cocot} to 52 \citep{vaishnav2025cognitive}.}  
Crucially, \ttmtext elevates SigLIP-L16 to the level of GPT-4.1 on ColorSwap and \textbf{enables SigLIP-B16 to surpass GPT-4.1 on MMVP-VLM, establishing a new state of the art}.  
These results demonstrate that \ttmtext is a powerful and practical approach for enhancing model performance through self-improvement at test time.

\subsection{\ttmtext also improves generative multimodal models}
\label{sec:smolvlm}
Beyond contrastive vision-language models, we also apply \ttmtext (\cref{alg:ttm}) to a small yet capable generative multimodal model, SmolVLM-256M-Instruct \citep{marafioti2025smolvlm}. We compute the similarity score using \vqascoretext \citep{lin2024evaluating} with the prompt: \texttt{<image>} Does this image show ``\texttt{<text>}''? Please answer ``Yes'' or ``No''.
During training, the pseudo-labeled dataset includes both ``Yes'' and ``No'' responses.
The ``No'' examples are constructed from mismatched image-caption pairs within the same group, providing hard negatives without requiring additional data sources. 
Results are provided in \cref{tab:smolvlm}.
Although these benchmarks are primarily designed for contrastive vision-language models (e.g., their short captions are not always natural for generative MLLMs), \ttmtext still yields clear gains over \simplematchtext, with the significant improvements on MMVP-VLM and ColorSwap.

\begin{figure}[t]
  \centering
  \includegraphics[width=.9\linewidth]{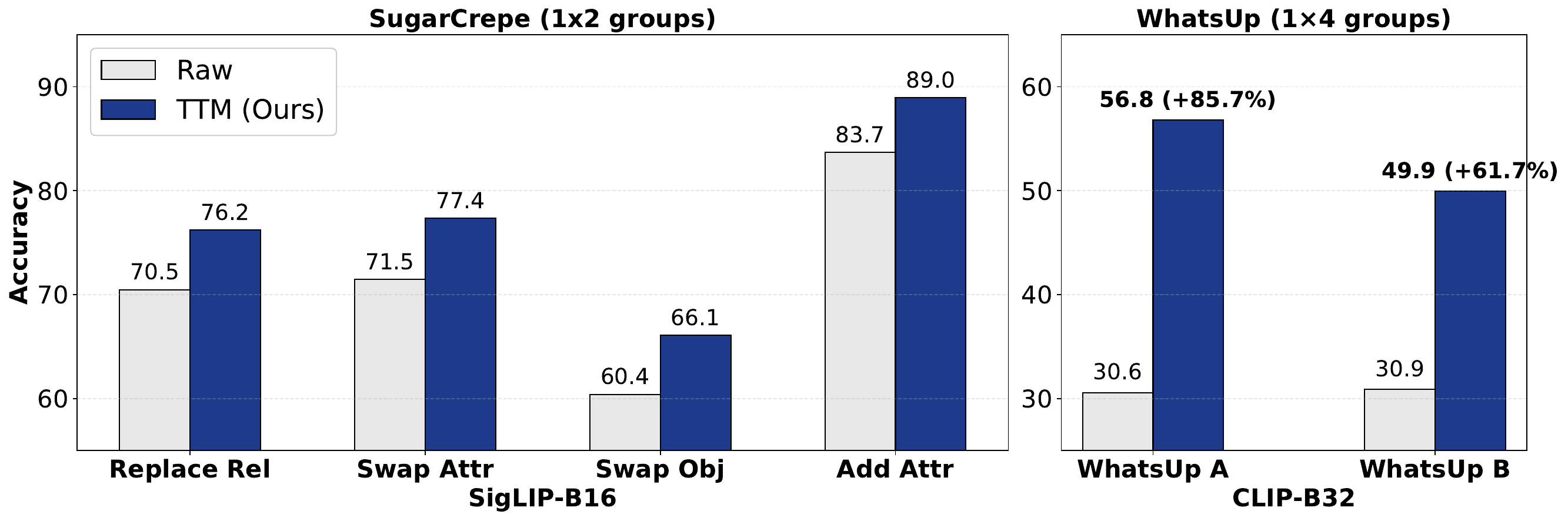}
  \caption{\ttmtext results on benchmarks without metric-induced boosts: for $1 \times k$ groups, \matchscoretext (and thus \simplematchtext) coincide with \groupscoretext. \emph{Left:} results on four SugarCrepe subsets consisting of $1 \times 2$ groups. \emph{Middle:} results on both \whatsup subsets consisting of $1 \times 4$ groups.}
  \label{fig:exp_1xk}
\end{figure}

\begin{table}[t]
\centering
  \caption{Performance of SmolVLM-256M-Instruct on Winoground, MMVP-VLM, and ColorSwap.
  Raw model performance is reported under \groupscoretext,
\simplematchtext corresponds to the performance under \matchscoretext (\cref{sec:metric}), 
and \ttmtext corresponds to the performance of \cref{alg:ttm}.
  We report absolute gains ($\Delta$), relative gains, and relative error reduction of \ttmtext over \simplematchtext.}
  \label{tab:smolvlm}
\begin{tabular}{l c c c c c}
\toprule
{Datasets} & {{SmolVLM}} & {\simplematchtext} & {+\,\ttmtext} & $\Delta$ & {Error  Red.} \\
\midrule
Winoground
  & 7.25
  & 61.75
  & \cellcolor{lightblue}\textbf{63.38\,\scriptsize$\pm$\,0.67}
  & \cellcolor{lightblue}{+}\,\textbf{1.6} \, (\textbf{2.6\%} $\uparrow$)
  & \cellcolor{lightblue}\textbf{4.3\%} $\downarrow$ \\
MMVP-VLM
  & 20.00
  & 76.30
  & \cellcolor{lightblue}\textbf{81.67\,\scriptsize$\pm$\,1.52}
  & \cellcolor{lightblue}{+}\,\textbf{5.4} \, (\textbf{7.0\%} $\uparrow$)
  & \cellcolor{lightblue}\textbf{22.7\%} $\downarrow$ \\
ColorSwap
  & 30.00
  & 80.00
  & \cellcolor{lightblue}\textbf{85.17\,\scriptsize$\pm$\,1.09}
  & \cellcolor{lightblue}{+}\,\textbf{5.2} \, (\textbf{6.5\%} $\uparrow$)
  & \cellcolor{lightblue}\textbf{25.9\%} $\downarrow$ \\
\bottomrule
\end{tabular}
\end{table}

\subsection{\ttmtext remains effective without metric-induced boosts}
\label{sec:exp_without_metric_boost}

To evaluate the effectiveness of \cref{alg:ttm} beyond cases where alternative metrics can inflate performance, we consider benchmarks with $1 \times k$ group structure, where \groupscoretext and \matchscoretext coincide and thus provide no metric-induced boost.

We experiment on 4 SugarCrepe subsets ($1 \times 2$ groups) and all 2 \whatsup subsets ($1 \times 4$ groups), reporting results in \cref{fig:exp_1xk}. 
Even without metric-induced gains, \cref{alg:ttm} consistently delivers substantial test-time improvements. 
The gains are especially striking on the \whatsup datasets, where \textbf{performance improves by up to 85.7\%}, turning these previously challenging tasks into tractable ones.  

Following \citet{li2024exploring}, we further convert the \whatsup datasets into 4 directional variants with $2 \times 2$ group structures. 
As shown in \cref{tab:whatsup_variants} (in \cref{app:exp_results}), \cref{alg:ttm} again yields significant improvements---\textbf{up to 135.1\% relative gains and 95.5\% relative error reduction}---on top of \simplematchtext. 
Together, these results demonstrate that \ttmtext is broadly effective across both $k \times k$ and $1 \times k$ groups, even when metric-induced effects are absent, as in the case of $1 \times k$ groups.

\subsection{\ttmtext extends beyond group-structured datasets}
\label{sec:exp_global_matching}

\begin{table}[t]
\centering
  \caption{Performance on non-grouped variants of Winoground, MMVP-VLM, and ColorSwap.
  Raw model performance is reported under \groupscoretext, 
  \simplematchtext corresponds to the performance of global assignment defined in \cref{eq:global_matching}, 
  and \ttmtext corresponds to the performance of the global variant of \cref{alg:ttm}. 
  We report absolute gains ($\Delta$), relative gains, and relative error reduction of \ttmtext over \simplematchtext.}
  \label{tab:global}
\begin{tabular}{l c c c c c}
\toprule
{Datasets} & {SigLIP-B16} & {\simplematchtext} & {+\,\ttmtext} & $\Delta$ & {Error  Red.  } \\
\midrule
Winoground & 10.25 & 44.38 & \cellcolor{lightblue}\textbf{46.78\,\scriptsize$\pm$\,1.05} 
  & \cellcolor{lightblue}{+}\,\textbf{2.4} \, (\textbf{5.4\%} $\uparrow$) 
  & \cellcolor{lightblue}\textbf{4.3\%} $\downarrow$ \\
MMVP-VLM  & 22.96 & 39.63 & \cellcolor{lightblue}\textbf{44.54\,\scriptsize$\pm$\,2.02} 
  & \cellcolor{lightblue}{+}\,\textbf{4.9} \, (\textbf{12.4\%} $\uparrow$) 
  & \cellcolor{lightblue}\textbf{8.1\%} $\downarrow$ \\
ColorSwap & 30.33 & 88.00 & \cellcolor{lightblue}\textbf{92.00\,\scriptsize$\pm$\,1.24} 
  & \cellcolor{lightblue}{+}\,\textbf{4.0} \, (\textbf{4.5\%} $\uparrow$) 
  & \cellcolor{lightblue}\textbf{33.3\%} $\downarrow$ \\
\bottomrule
\end{tabular}
\end{table}

To further assess the generality of \cref{alg:ttm}, we evaluate its global variant introduced in \cref{sec:global_matching} on datasets \emph{without any predefined group structures.} 
Specifically, we flatten Winoground, MMVP-VLM, and ColorSwap by removing local $k \times k$ groups, resulting in a general dataset with an image set $\cS_I$ and a caption set $\cS_C$. 
\looseness=-1

We report three metrics: (i) raw \groupscoretext (with the extra knowledge of the group structure), (ii) global assignment accuracy obtained via \simplematchtext under \cref{eq:global_matching}, and (iii) \ttmtext-boosted performance achieved using the global variant introduced in \cref{sec:global_matching}. 
Results show that even global assignment without group structures substantially outperforms the vanilla \groupscoretext, demonstrating the effectiveness of using \emph{matching-based supervision} to generate high-quality pseudo-labels.
More importantly, applying the iterative global TTM algorithm yields further gains over global assignment alone, with especially large relative error reductions on ColorSwap, i.e., \textbf{33.3\% relative error reduction on ColorSwap} (see \cref{tab:global}). This demonstrates that the test-time matching principle generalizes effectively beyond group-structured datasets.

\subsection{Analyses and ablations}
\label{sec:analyses}

\paragraph{Group matching provides strong supervision signals.}
The key advantage of \matchscoretext over \groupscoretext lies in its ability to leverage matching within local groups.
To assess the benefits of matching and group structure, we examine the raw performance of CLIP-B16 and SigLIP-B16 under different evaluation metrics.
In addition to \groupscoretext and \matchscoretext, we consider (i) \emph{global matching under \cref{eq:global_matching}}, which performs matching but ignores group structure, and (ii) \emph{individual matching within groups}, which preserves group structure but doesn't perform matching: it assign captions to images independently within the group. 
As shown in the left plot of \cref{fig:analyses},
\matchscoretext provides the strongest supervision signal among all metrics, making it most effective for guiding pseudo-labeling.

\paragraph{Skyline performance with oracle matching.}
To study the full potential of \ttmtext, we evaluate an oracle variant that incorporates pseudo-labels into $\cS_t$ if and only if they are correct (i.e., with oracle access). 
As shown in the middle plot of \cref{fig:analyses}, this oracle variant enables \ttmtext to bootstrap more aggressively, approaching human-level performance on Winoground. 
This suggests that improving pseudo-label quality---potentially through the incorporation of external supervision---could further enhance the effectiveness of \ttmtext.

\begin{figure}[t]
  \centering
  \includegraphics[width=.30\linewidth]{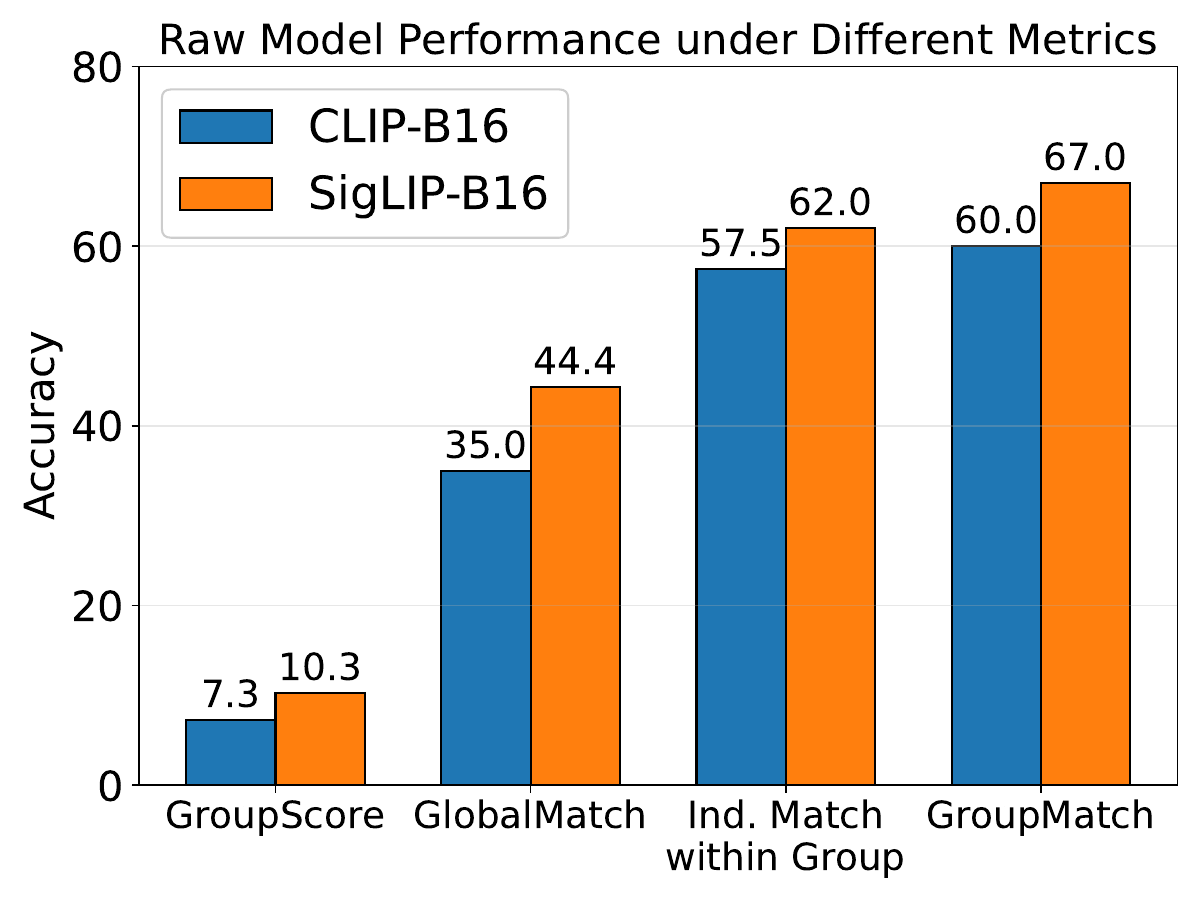}\hfill
  \includegraphics[width=.30\linewidth]{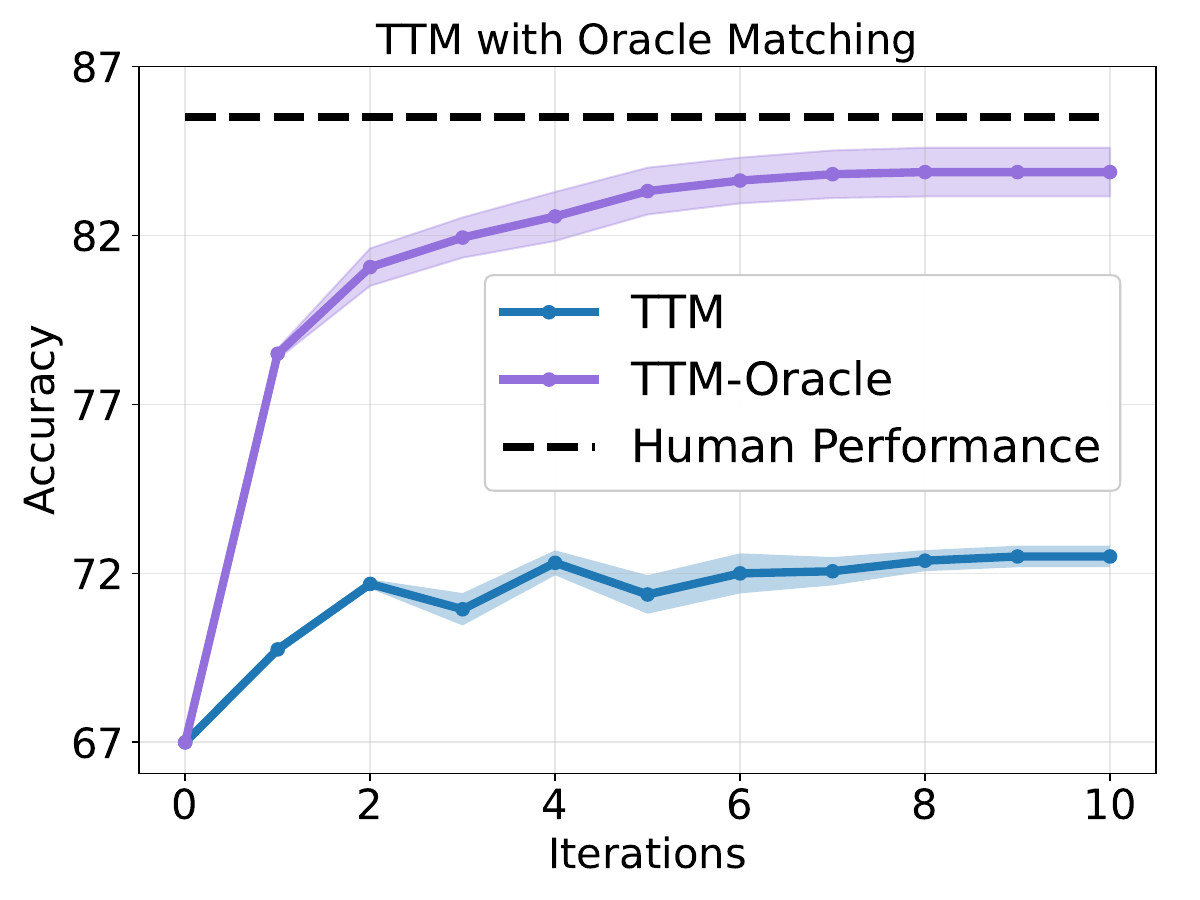}\hfill
  \includegraphics[width=.30\linewidth]{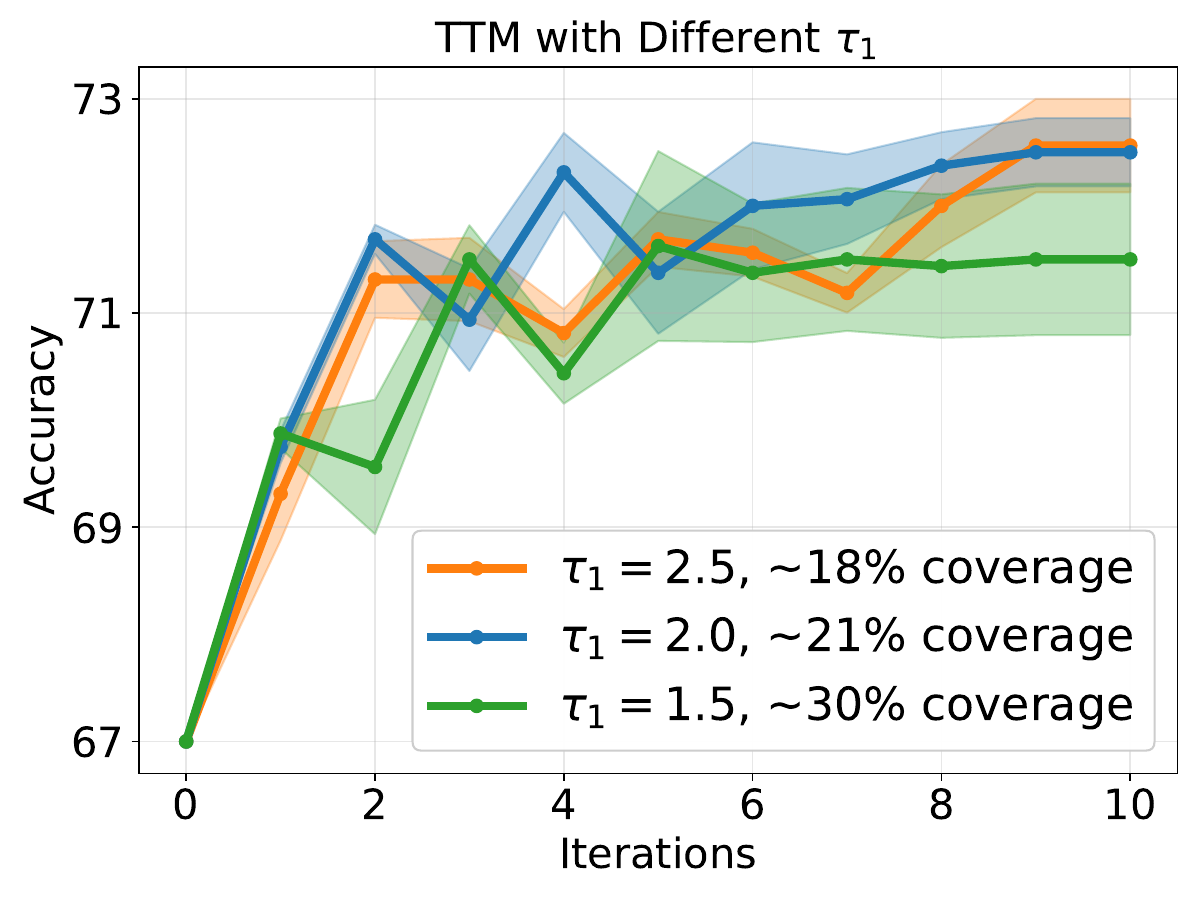}\hfill
  \caption{\emph{Left:} Raw performance of CLIP-B16 and SigLIP-B16 on Winoground under different evaluation metrics. 
  \emph{Middle:} Skyline performance of \ttmtext with oracle matching on Winoground with SigLIP-B16, illustrating the upper bound achievable by \ttmtext. 
  \emph{Right:} Effect of the initial threshold $\tau_1$ on \ttmtext performance, evaluated on Winoground with SigLIP-B16.}
  \label{fig:analyses}
\end{figure}

\paragraph{Threshold selection for \ttmtext.}
As discussed in \cref{sec:ttm}, we adopt a decaying threshold schedule that begins with high-quality pseudo-labels and gradually expands coverage. 
In our experiments, the final threshold $\tau_T$ is set to either $0$ (full coverage) or $0.1$ (typically covering more than 90\% of the data). 
The optimal initial threshold $\tau_1$ is more dataset- and model-dependent.
If a training set or hold-out split is available, $\tau_1$ can be selected based on matching results on that data (e.g., see the left and middle plots of \cref{fig:match_threshold}).
Otherwise, we find it effective to set $\tau_1$ such that roughly 15\%–30\% of the groups are matched initially.
The right plot of \cref{fig:analyses} shows \ttmtext results on Winoground with SigLIP-B16 for $\tau_1 \in \{2.5, 2, 1.5\}$, corresponding to roughly $\{18\%, 21\%, 30\%\}$ initial coverage. 
While performance varies slightly across these choices, all yield consistent gains, highlighting that \ttmtext robustly improves model performance at test time.
For the global matching variant, we find it effective to set $\tau_1$ such that about 50\% of the data are pseudo-labeled initially. 
See \cref{app:exp_hyper} for further discussion and complete hyperparameter settings used in our experiments.

\section{Related work}
\label{sec:related}

\paragraph{Compositional reasoning and evaluation metrics.}
Contrastive vision-language models (VLMs) such as CLIP \citep{radford2021learning} and SigLIP \citep{zhai2023sigmoid}, and multimodal large language models (MLLMs) such as the GPT \citep{achiam2023gpt, hurst2024gpt} and Gemini \citep{team2023gemini, comanici2025gemini} series, have achieved remarkable progress across a wide range of multimodal tasks.  
Yet both VLMs and MLLMs struggle on benchmarks specifically designed to test \emph{compositional reasoning}---the ability to systematically combine objects, attributes, and relations to interpret or reason about novel configurations \citep{lake2017building, bahdanau2018systematic, thrush2022winoground, hsieh2023sugarcrepe, kamath2023s, tong2024eyes, burapacheep2024colorswap}.  
These benchmarks are typically organized into small groups of images and captions that differ in subtle but systematic ways (e.g., captions with identical words but different orderings).  
The prevailing evaluation metric, the \groupscoretext, requires models to correctly assign each image to its corresponding caption and each caption to its corresponding image via isolated pairwise comparisons.  
While rigorous, this metric is also unforgiving: raw model performance often falls at or below random guessing \citep{thrush2022winoground, diwan2022winoground, tong2024eyes, burapacheep2024colorswap, li2024exploring}.  

Despite recent attempts to improve compositional reasoning in frontier multimodal models \citep{wu2023role, zhang2024vipact, vaishnav2025cognitive}, progress remains modest.  
For instance, the previous state of the art on Winoground---achieved by scaffolding and prompt tuning GPT-4V \citep{wu2023role, vaishnav2025cognitive}---was only 58.75, still well below the estimated human performance of 85.5 \citep{thrush2022winoground}.  
\looseness=-1

Our work takes a complementary perspective to prior efforts by revisiting the evaluation metrics used in compositional reasoning.
We introduce a \emph{group matching score} (\matchscoretext) that evaluates the best overall matching rather than isolated pairwise comparisons, revealing substantial hidden capability in both VLMs and MLLMs.  
Crucially, by simply overfitting to the induced matchings at test time, this hidden capability transfers into higher scores under the original \groupscoretext, closing much of the reported gap.  
With this adjustment, GPT-4.1 improves from 69.75 to 91.38 on Winoground---\emph{yielding the first result to surpass the estimated human performance of 85.5}.  
This finding echoes broader observations that measured capability can be highly sensitive to the choice of evaluation metric \citep{schaeffer2023emergent}, underscoring the need for continued research on evaluation protocols for frontier models.

\paragraph{Test-time training, pseudo-labeling, and adaptive schedules.}
Test-time training adapts models during inference to improve performance, with roots in early work on local learning and instance-specific adaptation \citep{cleveland1979robust, cleveland1988locally, bottou1992local, atkeson1997locally}.
The idea has regained attention in the era of large pretrained models, where test-time self-supervision can enhance performance without additional labeled data \citep{sun2020test, wang2020tent, gandelsman2022test, chen2022contrastive}.
Recent studies show that finetuning on retrieved data based on test prompts can significantly improve large language models \citep{hardt2024testtime, hubotter2025efficiently}, and test-time training has become a key component in tackling reasoning-heavy benchmarks such as ARC \citep{chollet2019measure, chollet2024arc, akyurek2025the}.

Our test-time matching algorithm (\ttmtext) shares this motivation but differs in key aspects.
Most prior methods adapt to each test instance independently, producing per-instance finetuned models and often relying on instance-specific in-context examples \citep{akyurek2025the}.
In contrast, \ttmtext leverages \matchscoretext-induced pseudo-labels across the \emph{entire test set}, iteratively updating a single model through an adaptive thresholding schedule.
This connects naturally to the literature on self-training \citep{kumar2020understanding} and semi-supervised learning \citep{zhu2005semi, chapelle2009semi, sohn2020fixmatch, zhang2021flexmatch, zhang2024labelbench}, where pseudo-labels drive improvements.
A central contribution of our approach is to exploit matching and group structure---both locally (\cref{sec:ttm}) and globally (\cref{sec:global_matching})---to generate high-quality pseudo-labels.

Finally, our adaptive thresholding schedule resonates with classical ideas in active learning \citep{castro2007minimax, balcan2007margin, dasgupta2009analysis, hanneke2014theory, krishnamurthy2019active, puchkin2021exponential, zhu2022active, zhu2022efficient}, though with reversed logic: whereas active learning typically queries the most uncertain examples for human annotation, our approach begins with the most confident pseudo-labels and gradually relaxes thresholds to expand coverage.
This confidence-first perspective is central to the effectiveness of \ttmtext, enabling consistent performance gains without any external supervision.

\section{Discussion}
\label{sec:discussion}

This work revisits the long-standing puzzle of compositional reasoning, where modern multimodal models often appear to perform no better than random guessing \citep{thrush2022winoground, diwan2022winoground, tong2024eyes, burapacheep2024colorswap, li2024exploring}.  
We show that this apparent limitation partly arises from the evaluation metrics themselves, which systematically underestimate model capability.  
We introduce \matchscoretext as a metric correction that yields a more faithful evaluation and can be translated back to the standard metric via a simple overfitting step; under this correction, GPT-4.1 surpasses estimated human performance on Winoground.
Building on this insight, we propose \emph{Test-Time Matching} (\ttmtext), an iterative, self-improving algorithm that further bootstraps model performance without external supervision.  
\ttmtext enables SigLIP-B16 to outperform GPT-4.1 on MMVP-VLM, establishing a new state of the art.  
Experiments across 16 dataset variants demonstrate that \ttmtext consistently improves performance across diverse settings, including those without metric-induced effects or predefined group structures.

Moving forward, we highlight two promising directions:  

\begin{itemize}[leftmargin=10pt, itemindent=*]
  \item \textbf{Recalibrating model evaluation.}  
  The same model on the same dataset can yield vastly different results under different metrics.  
  This underscores the need for more robust, transparent, and reliable evaluation protocols for compositional reasoning and beyond \citep{schaeffer2023emergent}.  

  \item \textbf{Extending \ttmtext beyond compositional reasoning.}  
  While developed in the context of compositional reasoning, the core principle of \ttmtext---iterative, matching-based self-training at test time---is general.  
  A natural next step is to explore this idea in broader multimodal or language settings.
  \looseness=-1
\end{itemize}

\section*{Author contributions}
YZ conceived the project, developed the algorithms, performed the majority of the implementation and experiments, and wrote the manuscript.
JZ and FT assisted with the implementation; JZ additionally conducted experiments on the WhatsUp datasets.

\section*{Acknowledgments}
YZ acknowledges support from NSF IIS 2425006.

\newpage
\bibliography{refs}

\begin{thebibliography}{52}
\providecommand{\natexlab}[1]{#1}
\providecommand{\url}[1]{\texttt{#1}}
\expandafter\ifx\csname urlstyle\endcsname\relax
  \providecommand{\doi}[1]{doi: #1}\else
  \providecommand{\doi}{doi: \begingroup \urlstyle{rm}\Url}\fi

\bibitem[Achiam et~al.(2023)Achiam, Adler, Agarwal, Ahmad, Akkaya, Aleman, Almeida, Altenschmidt, Altman, Anadkat, et~al.]{achiam2023gpt}
Josh Achiam, Steven Adler, Sandhini Agarwal, Lama Ahmad, Ilge Akkaya, Florencia~Leoni Aleman, Diogo Almeida, Janko Altenschmidt, Sam Altman, Shyamal Anadkat, et~al.
\newblock Gpt-4 technical report.
\newblock \emph{arXiv preprint arXiv:2303.08774}, 2023.

\bibitem[Aky{\"u}rek et~al.(2025)Aky{\"u}rek, Damani, Zweiger, Qiu, Guo, Pari, Kim, and Andreas]{akyurek2025the}
Ekin Aky{\"u}rek, Mehul Damani, Adam Zweiger, Linlu Qiu, Han Guo, Jyothish Pari, Yoon Kim, and Jacob Andreas.
\newblock The surprising effectiveness of test-time training for few-shot learning.
\newblock In \emph{Forty-second International Conference on Machine Learning}, 2025.

\bibitem[Atkeson et~al.(1997)Atkeson, Moore, and Schaal]{atkeson1997locally}
Christopher~G Atkeson, Andrew~W Moore, and Stefan Schaal.
\newblock Locally weighted learning.
\newblock \emph{Artificial intelligence review}, 11\penalty0 (1):\penalty0 11--73, 1997.

\bibitem[Bahdanau et~al.(2019)Bahdanau, Murty, Noukhovitch, Nguyen, de~Vries, and Courville]{bahdanau2018systematic}
Dzmitry Bahdanau, Shikhar Murty, Michael Noukhovitch, Thien~Huu Nguyen, Harm de~Vries, and Aaron Courville.
\newblock Systematic generalization: What is required and can it be learned?
\newblock In \emph{International Conference on Learning Representations}, 2019.

\bibitem[Balcan et~al.(2007)Balcan, Broder, and Zhang]{balcan2007margin}
Maria-Florina Balcan, Andrei Broder, and Tong Zhang.
\newblock Margin based active learning.
\newblock In \emph{International Conference on Computational Learning Theory}, pages 35--50. Springer, 2007.

\bibitem[Bottou and Vapnik(1992)]{bottou1992local}
L{\'e}on Bottou and Vladimir Vapnik.
\newblock Local learning algorithms.
\newblock \emph{Neural computation}, 4\penalty0 (6):\penalty0 888--900, 1992.

\bibitem[Burapacheep et~al.(2024)Burapacheep, Gaur, Bhatia, and Thrush]{burapacheep2024colorswap}
Jirayu Burapacheep, Ishan Gaur, Agam Bhatia, and Tristan Thrush.
\newblock Colorswap: A color and word order dataset for multimodal evaluation.
\newblock \emph{arXiv preprint arXiv:2402.04492}, 2024.

\bibitem[Castro and Nowak(2007)]{castro2007minimax}
Rui~M Castro and Robert~D Nowak.
\newblock Minimax bounds for active learning.
\newblock In \emph{International Conference on Computational Learning Theory}, pages 5--19. Springer, 2007.

\bibitem[Chapelle et~al.(2009)Chapelle, Scholkopf, and Zien]{chapelle2009semi}
Olivier Chapelle, Bernhard Scholkopf, and Alexander Zien.
\newblock Semi-supervised learning (chapelle, o. et al., eds.; 2006)[book reviews].
\newblock \emph{IEEE Transactions on Neural Networks}, 20\penalty0 (3):\penalty0 542--542, 2009.

\bibitem[Chen et~al.(2022)Chen, Wang, Darrell, and Ebrahimi]{chen2022contrastive}
Dian Chen, Dequan Wang, Trevor Darrell, and Sayna Ebrahimi.
\newblock Contrastive test-time adaptation.
\newblock In \emph{Proceedings of the IEEE/CVF conference on computer vision and pattern recognition}, pages 295--305, 2022.

\bibitem[Chollet(2019)]{chollet2019measure}
Fran{\c{c}}ois Chollet.
\newblock On the measure of intelligence.
\newblock \emph{arXiv preprint arXiv:1911.01547}, 2019.

\bibitem[Chollet et~al.(2024)Chollet, Knoop, Kamradt, and Landers]{chollet2024arc}
Francois Chollet, Mike Knoop, Gregory Kamradt, and Bryan Landers.
\newblock Arc prize 2024: Technical report.
\newblock \emph{arXiv preprint arXiv:2412.04604}, 2024.

\bibitem[Cleveland(1979)]{cleveland1979robust}
William~S Cleveland.
\newblock Robust locally weighted regression and smoothing scatterplots.
\newblock \emph{Journal of the American statistical association}, 74\penalty0 (368):\penalty0 829--836, 1979.

\bibitem[Cleveland and Devlin(1988)]{cleveland1988locally}
William~S Cleveland and Susan~J Devlin.
\newblock Locally weighted regression: an approach to regression analysis by local fitting.
\newblock \emph{Journal of the American statistical association}, 83\penalty0 (403):\penalty0 596--610, 1988.

\bibitem[Comanici et~al.(2025)Comanici, Bieber, Schaekermann, Pasupat, Sachdeva, Dhillon, Blistein, Ram, Zhang, Rosen, et~al.]{comanici2025gemini}
Gheorghe Comanici, Eric Bieber, Mike Schaekermann, Ice Pasupat, Noveen Sachdeva, Inderjit Dhillon, Marcel Blistein, Ori Ram, Dan Zhang, Evan Rosen, et~al.
\newblock Gemini 2.5: Pushing the frontier with advanced reasoning, multimodality, long context, and next generation agentic capabilities.
\newblock \emph{arXiv preprint arXiv:2507.06261}, 2025.

\bibitem[Dasgupta et~al.(2009)Dasgupta, Kalai, and Tauman]{dasgupta2009analysis}
Sanjoy Dasgupta, Adam~Tauman Kalai, and Adam Tauman.
\newblock Analysis of perceptron-based active learning.
\newblock \emph{Journal of Machine Learning Research}, 10\penalty0 (2), 2009.

\bibitem[Diwan et~al.(2022)Diwan, Berry, Choi, Harwath, and Mahowald]{diwan2022winoground}
Anuj Diwan, Layne Berry, Eunsol Choi, David Harwath, and Kyle Mahowald.
\newblock Why is winoground hard? investigating failures in visuolinguistic compositionality.
\newblock \emph{arXiv preprint arXiv:2211.00768}, 2022.

\bibitem[Gandelsman et~al.(2022)Gandelsman, Sun, Chen, and Efros]{gandelsman2022test}
Yossi Gandelsman, Yu~Sun, Xinlei Chen, and Alexei Efros.
\newblock Test-time training with masked autoencoders.
\newblock \emph{Advances in Neural Information Processing Systems}, 35:\penalty0 29374--29385, 2022.

\bibitem[Hanneke(2014)]{hanneke2014theory}
Steve Hanneke.
\newblock Theory of active learning.
\newblock \emph{Foundations and Trends in Machine Learning}, 7\penalty0 (2-3), 2014.

\bibitem[Hardt and Sun(2024)]{hardt2024testtime}
Moritz Hardt and Yu~Sun.
\newblock Test-time training on nearest neighbors for large language models.
\newblock In \emph{The Twelfth International Conference on Learning Representations}, 2024.

\bibitem[Hsieh et~al.(2023)Hsieh, Zhang, Ma, Kembhavi, and Krishna]{hsieh2023sugarcrepe}
Cheng-Yu Hsieh, Jieyu Zhang, Zixian Ma, Aniruddha Kembhavi, and Ranjay Krishna.
\newblock Sugarcrepe: Fixing hackable benchmarks for vision-language compositionality.
\newblock \emph{Advances in neural information processing systems}, 36:\penalty0 31096--31116, 2023.

\bibitem[Hu et~al.(2022)Hu, Shen, Wallis, Allen-Zhu, Li, Wang, Wang, Chen, et~al.]{hu2022lora}
Edward~J Hu, Yelong Shen, Phillip Wallis, Zeyuan Allen-Zhu, Yuanzhi Li, Shean Wang, Liang Wang, Weizhu Chen, et~al.
\newblock Lora: Low-rank adaptation of large language models.
\newblock \emph{Iclr}, 1\penalty0 (2):\penalty0 3, 2022.

\bibitem[H{\"u}botter et~al.(2025)H{\"u}botter, Bongni, Hakimi, and Krause]{hubotter2025efficiently}
Jonas H{\"u}botter, Sascha Bongni, Ido Hakimi, and Andreas Krause.
\newblock Efficiently learning at test-time: Active fine-tuning of {LLM}s.
\newblock In \emph{The Thirteenth International Conference on Learning Representations}, 2025.

\bibitem[Hurst et~al.(2024)Hurst, Lerer, Goucher, Perelman, Ramesh, Clark, Ostrow, Welihinda, Hayes, Radford, et~al.]{hurst2024gpt}
Aaron Hurst, Adam Lerer, Adam~P Goucher, Adam Perelman, Aditya Ramesh, Aidan Clark, AJ~Ostrow, Akila Welihinda, Alan Hayes, Alec Radford, et~al.
\newblock Gpt-4o system card.
\newblock \emph{arXiv preprint arXiv:2410.21276}, 2024.

\bibitem[Kamath et~al.(2023)Kamath, Hessel, and Chang]{kamath2023s}
Amita Kamath, Jack Hessel, and Kai-Wei Chang.
\newblock What's" up" with vision-language models? investigating their struggle with spatial reasoning.
\newblock \emph{arXiv preprint arXiv:2310.19785}, 2023.

\bibitem[Krishnamurthy et~al.(2019)Krishnamurthy, Agarwal, Huang, Daum{\'e}~III, and Langford]{krishnamurthy2019active}
Akshay Krishnamurthy, Alekh Agarwal, Tzu-Kuo Huang, Hal Daum{\'e}~III, and John Langford.
\newblock Active learning for cost-sensitive classification.
\newblock \emph{Journal of Machine Learning Research}, 20\penalty0 (65):\penalty0 1--50, 2019.

\bibitem[Kuhn(1955)]{kuhn1955hungarian}
Harold~W Kuhn.
\newblock The hungarian method for the assignment problem.
\newblock \emph{Naval research logistics quarterly}, 2\penalty0 (1-2):\penalty0 83--97, 1955.

\bibitem[Kumar et~al.(2020)Kumar, Ma, and Liang]{kumar2020understanding}
Ananya Kumar, Tengyu Ma, and Percy Liang.
\newblock Understanding self-training for gradual domain adaptation.
\newblock In \emph{International conference on machine learning}, pages 5468--5479. PMLR, 2020.

\bibitem[Lake et~al.(2017)Lake, Ullman, Tenenbaum, and Gershman]{lake2017building}
Brenden~M Lake, Tomer~D Ullman, Joshua~B Tenenbaum, and Samuel~J Gershman.
\newblock Building machines that learn and think like people.
\newblock \emph{Behavioral and brain sciences}, 40:\penalty0 e253, 2017.

\bibitem[Li et~al.(2025)Li, Koh, and Du]{li2024exploring}
Siting Li, Pang~Wei Koh, and Simon~Shaolei Du.
\newblock Exploring how generative mllms perceive more than clip with the same vision encoder.
\newblock \emph{Proceedings of the 63rd Annual Meeting of the Association for Computational Linguistics}, 2025.

\bibitem[Lin et~al.(2024)Lin, Pathak, Li, Li, Xia, Neubig, Zhang, and Ramanan]{lin2024evaluating}
Zhiqiu Lin, Deepak Pathak, Baiqi Li, Jiayao Li, Xide Xia, Graham Neubig, Pengchuan Zhang, and Deva Ramanan.
\newblock Evaluating text-to-visual generation with image-to-text generation.
\newblock In \emph{European Conference on Computer Vision}, pages 366--384. Springer, 2024.

\bibitem[Loshchilov and Hutter(2017)]{loshchilov2017decoupled}
Ilya Loshchilov and Frank Hutter.
\newblock Decoupled weight decay regularization.
\newblock \emph{arXiv preprint arXiv:1711.05101}, 2017.

\bibitem[Marafioti et~al.(2025)Marafioti, Zohar, Farré, Noyan, Bakouch, Cuenca, Zakka, Allal, Lozhkov, Tazi, Srivastav, Lochner, Larcher, Morlon, Tunstall, von Werra, and Wolf]{marafioti2025smolvlm}
Andrés Marafioti, Orr Zohar, Miquel Farré, Merve Noyan, Elie Bakouch, Pedro Cuenca, Cyril Zakka, Loubna~Ben Allal, Anton Lozhkov, Nouamane Tazi, Vaibhav Srivastav, Joshua Lochner, Hugo Larcher, Mathieu Morlon, Lewis Tunstall, Leandro von Werra, and Thomas Wolf.
\newblock Smolvlm: Redefining small and efficient multimodal models.
\newblock \emph{arXiv preprint arXiv:2504.05299}, 2025.

\bibitem[Puchkin and Zhivotovskiy(2021)]{puchkin2021exponential}
Nikita Puchkin and Nikita Zhivotovskiy.
\newblock Exponential savings in agnostic active learning through abstention.
\newblock In \emph{Conference on learning theory}, pages 3806--3832. PMLR, 2021.

\bibitem[Radford et~al.(2021)Radford, Kim, Hallacy, Ramesh, Goh, Agarwal, Sastry, Askell, Mishkin, Clark, et~al.]{radford2021learning}
Alec Radford, Jong~Wook Kim, Chris Hallacy, Aditya Ramesh, Gabriel Goh, Sandhini Agarwal, Girish Sastry, Amanda Askell, Pamela Mishkin, Jack Clark, et~al.
\newblock Learning transferable visual models from natural language supervision.
\newblock In \emph{International conference on machine learning}, pages 8748--8763. PmLR, 2021.

\bibitem[Schaeffer et~al.(2023)Schaeffer, Miranda, and Koyejo]{schaeffer2023emergent}
Rylan Schaeffer, Brando Miranda, and Sanmi Koyejo.
\newblock Are emergent abilities of large language models a mirage?
\newblock \emph{Advances in neural information processing systems}, 36:\penalty0 55565--55581, 2023.

\bibitem[Sohn et~al.(2020)Sohn, Berthelot, Carlini, Zhang, Zhang, Raffel, Cubuk, Kurakin, and Li]{sohn2020fixmatch}
Kihyuk Sohn, David Berthelot, Nicholas Carlini, Zizhao Zhang, Han Zhang, Colin~A Raffel, Ekin~Dogus Cubuk, Alexey Kurakin, and Chun-Liang Li.
\newblock Fixmatch: Simplifying semi-supervised learning with consistency and confidence.
\newblock \emph{Advances in neural information processing systems}, 33:\penalty0 596--608, 2020.

\bibitem[Sun et~al.(2020)Sun, Wang, Liu, Miller, Efros, and Hardt]{sun2020test}
Yu~Sun, Xiaolong Wang, Zhuang Liu, John Miller, Alexei Efros, and Moritz Hardt.
\newblock Test-time training with self-supervision for generalization under distribution shifts.
\newblock In \emph{International conference on machine learning}, pages 9229--9248. PMLR, 2020.

\bibitem[Team et~al.(2023)Team, Anil, Borgeaud, Alayrac, Yu, Soricut, Schalkwyk, Dai, Hauth, Millican, et~al.]{team2023gemini}
Gemini Team, Rohan Anil, Sebastian Borgeaud, Jean-Baptiste Alayrac, Jiahui Yu, Radu Soricut, Johan Schalkwyk, Andrew~M Dai, Anja Hauth, Katie Millican, et~al.
\newblock Gemini: a family of highly capable multimodal models.
\newblock \emph{arXiv preprint arXiv:2312.11805}, 2023.

\bibitem[Thrush et~al.(2022)Thrush, Jiang, Bartolo, Singh, Williams, Kiela, and Ross]{thrush2022winoground}
Tristan Thrush, Ryan Jiang, Max Bartolo, Amanpreet Singh, Adina Williams, Douwe Kiela, and Candace Ross.
\newblock Winoground: Probing vision and language models for visio-linguistic compositionality.
\newblock In \emph{Proceedings of the IEEE/CVF Conference on Computer Vision and Pattern Recognition}, pages 5238--5248, 2022.

\bibitem[Tong et~al.(2024)Tong, Liu, Zhai, Ma, LeCun, and Xie]{tong2024eyes}
Shengbang Tong, Zhuang Liu, Yuexiang Zhai, Yi~Ma, Yann LeCun, and Saining Xie.
\newblock Eyes wide shut? exploring the visual shortcomings of multimodal llms.
\newblock In \emph{Proceedings of the IEEE/CVF Conference on Computer Vision and Pattern Recognition}, pages 9568--9578, 2024.

\bibitem[Vaishnav and Tammet(2025)]{vaishnav2025cognitive}
Mohit Vaishnav and Tanel Tammet.
\newblock A cognitive paradigm approach to probe the perception-reasoning interface in vlms.
\newblock \emph{arXiv preprint arXiv:2501.13620}, 2025.

\bibitem[Wang et~al.(2020)Wang, Shelhamer, Liu, Olshausen, and Darrell]{wang2020tent}
Dequan Wang, Evan Shelhamer, Shaoteng Liu, Bruno Olshausen, and Trevor Darrell.
\newblock Tent: Fully test-time adaptation by entropy minimization.
\newblock \emph{arXiv preprint arXiv:2006.10726}, 2020.

\bibitem[Wu et~al.(2023)Wu, Zhang, Xiong, Oguz, Gee, and Nie]{wu2023role}
Yifan Wu, Pengchuan Zhang, Wenhan Xiong, Barlas Oguz, James~C Gee, and Yixin Nie.
\newblock The role of chain-of-thought in complex vision-language reasoning task.
\newblock \emph{arXiv preprint arXiv:2311.09193}, 2023.

\bibitem[Zhai et~al.(2023)Zhai, Mustafa, Kolesnikov, and Beyer]{zhai2023sigmoid}
Xiaohua Zhai, Basil Mustafa, Alexander Kolesnikov, and Lucas Beyer.
\newblock Sigmoid loss for language image pre-training.
\newblock In \emph{Proceedings of the IEEE/CVF international conference on computer vision}, pages 11975--11986, 2023.

\bibitem[Zhang et~al.(2021)Zhang, Wang, Hou, Wu, Wang, Okumura, and Shinozaki]{zhang2021flexmatch}
Bowen Zhang, Yidong Wang, Wenxin Hou, Hao Wu, Jindong Wang, Manabu Okumura, and Takahiro Shinozaki.
\newblock Flexmatch: Boosting semi-supervised learning with curriculum pseudo labeling.
\newblock \emph{Advances in neural information processing systems}, 34:\penalty0 18408--18419, 2021.

\bibitem[Zhang et~al.(2024{\natexlab{a}})Zhang, Yang, Lyu, Jin, Yao, Chen, and Luo]{zhang2024cocot}
Daoan Zhang, Junming Yang, Hanjia Lyu, Zijian Jin, Yuan Yao, Mingkai Chen, and Jiebo Luo.
\newblock Cocot: Contrastive chain-of-thought prompting for large multimodal models with multiple image inputs.
\newblock \emph{arXiv preprint arXiv:2401.02582}, 2024{\natexlab{a}}.

\bibitem[Zhang et~al.(2024{\natexlab{b}})Zhang, Chen, Canal, Das, Bhatt, Mussmann, Zhu, Bilmes, Du, Jamieson, and Nowak]{zhang2024labelbench}
Jifan Zhang, Yifang Chen, Gregory Canal, Arnav~Mohanty Das, Gantavya Bhatt, Stephen Mussmann, Yinglun Zhu, Jeff Bilmes, Simon~Shaolei Du, Kevin Jamieson, and Robert Nowak.
\newblock Labelbench: A comprehensive framework for benchmarking adaptive label-efficient learning.
\newblock \emph{Journal of Data-centric Machine Learning Research}, 2024{\natexlab{b}}.

\bibitem[Zhang et~al.(2024{\natexlab{c}})Zhang, Rossi, Yu, Dernoncourt, Zhang, Gu, Kim, Chen, Wang, and Lipka]{zhang2024vipact}
Zhehao Zhang, Ryan Rossi, Tong Yu, Franck Dernoncourt, Ruiyi Zhang, Jiuxiang Gu, Sungchul Kim, Xiang Chen, Zichao Wang, and Nedim Lipka.
\newblock Vipact: Visual-perception enhancement via specialized vlm agent collaboration and tool-use.
\newblock \emph{arXiv preprint arXiv:2410.16400}, 2024{\natexlab{c}}.

\bibitem[Zhu(2005)]{zhu2005semi}
Xiaojin~Jerry Zhu.
\newblock Semi-supervised learning literature survey.
\newblock 2005.

\bibitem[Zhu and Nowak(2022{\natexlab{a}})]{zhu2022active}
Yinglun Zhu and Robert Nowak.
\newblock Active learning with neural networks: Insights from nonparametric statistics.
\newblock \emph{Advances in Neural Information Processing Systems}, 35:\penalty0 142--155, 2022{\natexlab{a}}.

\bibitem[Zhu and Nowak(2022{\natexlab{b}})]{zhu2022efficient}
Yinglun Zhu and Robert Nowak.
\newblock Efficient active learning with abstention.
\newblock \emph{Advances in Neural Information Processing Systems}, 35:\penalty0 35379--35391, 2022{\natexlab{b}}.

\end{thebibliography}
\bibliographystyle{plainnat}

\newpage
\appendix

\section{Proofs and supporting results from \cref{sec:methods}}
\label{app:support_methods}

\subsection{Proofs of \cref{prop:group_score_prob} and \cref{prop:match_score_prob}}
\label{app:proofs_group}

\propGroupScoreProb*

\begin{proof}
Because the entries of $s$ are i.i.d. sampled from a continuous distribution (here $\unif([0,1])$), ties occur with probability $0$, so we may use strict inequalities throughout.

Denote $d_i \ldef s_{ii}$ and, for $i\neq j$, set $m_{ij}\ldef \min\{d_i,d_j\}$. 
By the definition of the \groupscoretext, the event $\{\groupscoremath(s)=1\}$ is equivalent to requiring
$s_{ij}<m_{ij}$ and $s_{ji}<m_{ij}$ for every $i\neq j$.
Conditioning on the diagonal $d=(d_1,\dots,d_k)$ and using independence of the off-diagonal entries,
\begin{align*}
\bbP \big(\groupscoremath(s)=1 \mid d\big)
&= \prod_{i<j} \bbP(s_{ij}<m_{ij})\,\bbP(s_{ji}<m_{ij})
= \prod_{i<j} m_{ij}^{\,2}.
\end{align*}

Let $0 < x_1 < \cdots < x_k < 1$ be the order statistics of $(d_1,\dots,d_k)$.
We then have $m_{ij}=x_{\min\{r(i),r(j)\}}$, where $r(\cdot)$ is the rank, hence
\begin{align*}
\prod_{i<j} m_{ij}^{\,2}
= \prod_{a=1}^{k} x_a^{\,2(k-a)} .
\end{align*}
Since $(x_1,\dots,x_k)$ are the order statistics of i.i.d.\ $\unif([0,1])$ samples, their joint density is $k!$ on the ordered region $\{0 < x_1 < \cdots < x_k < 1\}$ (and $0$ elsewhere). Therefore,
\begin{align*}
\bbP \big(\groupscoremath(s)=1\big)
&= k!\int_{0 < x_1 < \cdots < x_k < 1}
\ \prod_{a=1}^{k} x_a^{\,2(k-a)}\,dx_1\cdots dx_k .
\end{align*}

For $1\leq \ell \leq k$ and $y\in(0,1]$, define
\begin{align*}
I_\ell(y)\ldef \int_{0<x_1<\cdots<x_\ell<y}
\ \prod_{a=1}^{\ell} x_a^{\,2(k-a)}\,dx_1\cdots dx_\ell .
\end{align*}
We claim that, for $\ell=1,\dots,k$,
\begin{align*}
I_\ell(y)=\frac{y^{\,\ell(2k-\ell)}}{\prod_{r=1}^{\ell} r(2k-r)} .
\end{align*}
This is proved by induction on $\ell$. For $\ell=1$,
\begin{align*}
I_1(y)=\int_0^y x^{2(k-1)}\,dx=\frac{y^{2k-1}}{2k-1}.
\end{align*}
Assume it holds for $\ell-1$. Then
\begin{align*}
I_\ell(y)
&=\int_0^y x_\ell^{\,2(k-\ell)}\,I_{\ell-1}(x_\ell)\,dx_\ell\\
&=\frac{1}{\prod_{r=1}^{\ell-1} r(2k-r)}
\int_0^y x_\ell^{\,2(k-\ell)+(\ell-1)(2k-(\ell-1))}\,dx_\ell\\
&=\frac{1}{\prod_{r=1}^{\ell-1} r(2k-r)}
\cdot \frac{y^{\,\ell(2k-\ell)}}{\ell(2k-\ell)},
\end{align*}
since $2(k-\ell)+(\ell-1)(2k-(\ell-1))=\ell(2k-\ell)-1$. 
Thus the claim holds. Taking $\ell=k$ and $y=1$ gives
\begin{align*}
\int_{0 < x_1 < \cdots < x_k < 1}
\ \prod_{a=1}^{k} x_a^{\,2(k-a)}\,dx_1\cdots dx_k
= I_k(1)=\frac{1}{\prod_{r=1}^{k} r(2k-r)} .
\end{align*}
Therefore,
\begin{align*}
\bbP \big(\groupscoremath(s)=1\big)
= k!\,\prod_{r=1}^{k}\frac{1}{r(2k-r)}
= \frac{(k-1)!}{(2k-1)!}.
\end{align*}
\end{proof}

\propMatchScoreProb*
\begin{proof}
There are ${k!}$ distinct injective matchings.
Since the random variables $\{s_{ij}\}$ are continuous, ties occur with probability $0$.
By symmetry, each injective matching is equally likely to achieve the maximum total similarity.
Hence, the probability that the ground-truth matching $\pi^\star$ attains the maximum is 
  $\frac{1}{k!}$.
\end{proof}

\subsection{Supporting results for general rectangular groups}
\label{app:rectangular}
Without loss of generality, we consider a group of $m$ images $\{I_i\}_{i=1}^m$ and $k$ captions $\{C_i\}_{i=1}^k$ with $m < k$.
We assume the ground-truth pairings is $\{(I_i, C_i)\}_{i=1}^m$ (hidden from the learner).
As in the main text, we study a random guessing model that assigns i.i.d. similarity scores
$s_{ij} \ldef \simmath(I_i,C_j) \sim \unif([0,1])$ for each pair $(I_i,C_j)$, and collect them into a similarity matrix $s \in \bbR^{m\times k}$.

\paragraph{\groupscoretext for $m \times k$ groups.}
Analogous to the $k\times k$ and $1 \times k$ cases, the \groupscoretext for $m \times k$ groups can be defined as 
\begin{align*}
\groupscoremath(s) \ldef
\begin{cases}
1 & 
  \forall \, i \in [m]:\; s_{ii} > \max_{j \neq i} s_{ij}, \\[6pt]
0 & \text{otherwise}.
\end{cases}
\end{align*}
Under the random guessing model, the probability of achieving a \groupscoretext of 1 for rectangular group is given below.
\begin{proposition}
  \label{prop:group_score_rectangular}
  For random similarity score $s \in \bbR^{m \times k}$,
  $\bbP(\groupscoremath(s) = 1) = \frac{1}{k^m}$.
\end{proposition}
\begin{proof}
  Since the random variables $\{s_{ij}\}$ are continuous, ties occur with probability $0$.  
  For each row $i$, 
by symmetry, 
the probability that $s_{ii}$ is the largest among the $k$ i.i.d. entries $\{s_{ij}\}_{j=1}^k$ is $1/k$.  
Since rows are independent, we have
  $$\bbP (\forall\, i \in [m]: s_{ii} > \max_{j \neq i} s_{ij}) = \prod_{i=1}^m \frac{1}{k} = \frac{1}{k^m}.$$
\end{proof}

\paragraph{\matchscoretext for $m \times k$ groups.}
We extend \matchscoretext to the general rectangular case by considering \emph{injective} matchings 
$\pi:[m]\rightarrow [k]$ (i.e., $\pi(i)\neq\pi(j)$ for $i\neq j$).
With the ground-truth injective matching $\pi^\star:i\mapsto i$, we define
\matchscoretext as  
\begin{align*}
  \matchscoremath(s) \ldef 
  \begin{cases}
    1 & \text{if } \sum_{i=1}^m s_{i,\pi^\star(i)} > \sum_{i=1}^m s_{i,\pi(i)}, \quad \forall\; \pi \neq \pi^\star, \\[3pt]
    0 & \text{otherwise}.
  \end{cases}
\end{align*}

Under the random guessing model, the probability of achieving a \matchscoretext of 1 for rectangular group is given below.
\begin{proposition}
\label{prop:match_score_rectangular}
  For random similarity scores $s \in \bbR^{m \times k}$, 
  $\bbP(\matchscoremath(s) = 1) = \frac{(k - m)!}{k!}$.
\end{proposition}
\begin{proof}
There are $\frac{k!}{(k-m)!}$ distinct injective matchings.
Since the random variables $\{s_{ij}\}$ are continuous, ties occur with probability $0$.
By symmetry, each injective matching is equally likely to achieve the maximum total similarity.
Hence, the probability that the ground-truth matching $\pi^\star$ attains the maximum is 
  $\big(\frac{k!}{(k - m)!}\big)^{-1} = \frac{(k - m)!}{k!}$.
\end{proof}

\paragraph{\matchscoretext helps for rectangular groups.}
For random similarity scores $s\in\bbR^{m\times k}$,
\begin{align*}
\bbP \big(\matchscoremath(s)=1\big)
= \frac{(k-m)!}{k!}
= \frac{1}{k(k-1)\cdots(k-m+1)}
\geq \frac{1}{k^m}
= \bbP\big(\groupscoremath(s)=1\big),
\end{align*}
with strict inequality for any $m\geq 2$ and equality for $m=1$ (\matchscoretext and \groupscoretext coincide when $m=1$).
Moreover, if the ground-truth injective matching $\pi^\star$ is identified, overfitting to the matching $\pi^\star$ at test time guarantees a \groupscoretext of $1$.
Thus, as in the square case, one can improve model performance under \groupscoretext via \simplematchtext: 
(i) selecting the most likely matching under \matchscoretext and (ii) 
overfitting to the matching at test time to transfer gains.

\section{Additional experimental details and results}
\label{app:exp}

\subsection{Implementation details and hyperparameters}
\label{app:exp_hyper}

\subsubsection{General training hyperparameters}
We summarize the general training setup and hyperparameter choices used in our experiments below.
Across all experiments, we use AdamW \citep{loshchilov2017decoupled} with $(\beta_1, \beta_2) = (0.9, 0.999)$ and weight decay $0.05$. 
The learning rate follows a cosine decay schedule and is restarted at each iteration with a multiplicative factor of $0.95$. 
Optimizer states are reset at each restart, with the exception of SigLIP-B16 on Winoground.

\begin{itemize}
  \item 
  For contrastive vision-language models, we set the number of iterations to $T = 10$ by default; additional results with fewer iterations are reported in \cref{app:runtime}. At each iteration, we train for 20 epochs by default, except on Winoground, where we train for 30 epochs.
We use a batch size of 50 for $2\times 2$ datasets and 100 for $1\times k$ datasets; the batch size is defined at the group level (e.g., 50 groups of size $2\times 2$ per batch).\footnote{We slightly increase the batch size when the total number of groups is just above a multiple of the default size. For instance, if the dataset contains 102 groups, we set the batch size to 51.}  

  \item For generative multimodal models, we set the number of iterations to $T=3$ and train for 10 epochs per iteration.
  We use a batch size of 16 for MMVP-VLM and ColorSwap datasets and a batch size of 32 for Winoground.
  The batch size here is also defined at the group level.
  Since the pseudo-labeled dataset includes both ``Yes'' and ``No'' responses (\cref{sec:smolvlm}), each $2 \times 2$ group provides 4 data points for training.
  \looseness=-1
\end{itemize}

By default, we do not apply data augmentation during training, as many datasets are designed to be sensitive to location or color.
However, we find it beneficial to apply a simple resizing (factor $1.1$) followed by random cropping for the following dataset-model pairs: Winoground with SigLIP-L16, MMVP-VLM with SigLIP-B16, ColorSwap with SigLIP-B16, MMVP-VLM with SigLIP-B16 under global matching, and CLIP-B32 with \whatsup A-Left-Right.

\subsubsection{\ttmtext-specific hyperparameters}

In \cref{tab:hyperparams,tab:hyperparams_smolvlm,tab:hyperparams_sugarcrepe_whatsup,tab:hyperparams_global}, we report, for each dataset-model pair, the initial threshold $\tau_1$, the final threshold $\tau_T$, the threshold decay schedule (linear or cosine), and the learning rate (lr).
For group matching (\cref{tab:hyperparams,tab:hyperparams_smolvlm,tab:hyperparams_sugarcrepe_whatsup}), we use absolute thresholds. 
For global matching (\cref{tab:hyperparams_global}), 
we adopt the percentile-based thresholding mentioned in \cref{sec:global_matching}: at iteration $t$, the top $1-\tau_t$ fraction of pseudo-labels (ranked by similarity) is selected.

In our experiments, the final threshold $\tau_T$ is set to either $0$ (full coverage) or $0.1$ (typically covering more than 90\% of the data). 
The initial threshold $\tau_1$ is more dataset- and model-dependent.
For group matching, we find it effective to set $\tau_1$ such that roughly 15\%–30\% of the groups are matched initially,
though in some cases we use thresholds outside this range when they yield better performance (e.g., higher selection fractions for ColorSwap with SigLIP models and lower fractions for \whatsup $2\times 2$ variants with CLIP-B32).
For global matching, performance tends to improve with a larger initial selection fraction---typically around 50\%.

\begin{table}[h]
\centering
  \caption{Hyperparameters used for experiments in \cref{sec:exp_2x2}.}
\label{tab:hyperparams}
\setlength{\tabcolsep}{8pt}
\begin{tabular}{l l c c c c}
\toprule
Dataset & Model & $\tau_1$ & $\tau_T$ & Schedule & lr \\
\midrule
\multirow{3}{*}{Winoground} 
  & CLIP-B16   & 0.9 & 0   & linear & $2.0 \times 10^{-5}$ \\
  & SigLIP-B16 & 2.0 & 0   & linear & $1.0 \times 10^{-5}$ \\
  & SigLIP-L16 & 2.0 & 0.1 & cosine & $4.0 \times 10^{-5}$ \\
\addlinespace[3pt]
\multirow{2}{*}{MMVP-VLM} 
  & CLIP-B16   & 2.0 & 0   & linear & $1.0 \times 10^{-5}$ \\
  & SigLIP-B16 & 2.0 & 0.1 & cosine & $2.0 \times 10^{-5}$ \\
\addlinespace[3pt]
\multirow{3}{*}{ColorSwap} 
  & CLIP-B16   & 2.3 & 0   & cosine & $4.0 \times 10^{-5}$ \\
  & SigLIP-B16 & 1.0 & 0   & cosine & $4.0 \times 10^{-5}$ \\
  & SigLIP-L16 & 2.5 & 0   & cosine & $4.0 \times 10^{-5}$ \\
\bottomrule
\end{tabular}
\end{table}

\begin{table}[h]
\centering
  \caption{Hyperparameters used for experiments in \cref{sec:smolvlm}.}
\label{tab:hyperparams_smolvlm}
\setlength{\tabcolsep}{8pt}
\begin{tabular}{l l c c c c}
\toprule
Dataset & Model & $\tau_1$ & $\tau_T$ & Schedule & lr \\
\midrule
Winoground & SmolVLM-256M-Instruct & 0.1 & 0   & linear & $4.0 \times 10^{-5}$ \\
MMVP-VLM   & SmolVLM-256M-Instruct & 0.1 & 0   & linear & $6.0 \times 10^{-5}$ \\
ColorSwap  & SmolVLM-256M-Instruct & 0.25 & 0   & linear & $6.0 \times 10^{-5}$ \\
\bottomrule
\end{tabular}
\end{table}

\begin{table}[h]
\centering
  \caption{Hyperparameters used for experiments in \cref{sec:exp_without_metric_boost}.}
\label{tab:hyperparams_sugarcrepe_whatsup}
\setlength{\tabcolsep}{8pt}
\begin{tabular}{l l c c c c}
\toprule
\textbf{Variant} & \textbf{Model} & $\tau_1$ & $\tau_T$ & Schedule & lr \\
\midrule
Replace Relation   & SigLIP-B16 & 2.1  & 0   & cosine & $1.0 \times 10^{-5}$ \\
Swap Attribute     & SigLIP-B16 & 1.8  & 0   & cosine & $1.0 \times 10^{-5}$ \\
Swap Object        & SigLIP-B16 & 2.0  & 0   & cosine & $1.0 \times 10^{-5}$ \\
Add Attribute      & SigLIP-B16 & 2.5  & 0   & cosine & $1.0 \times 10^{-5}$ \\
\addlinespace[3pt]
\whatsup A (1$\times$4) & CLIP-B32 & 0.55 & 0   & linear & $1.0 \times 10^{-5}$ \\
\whatsup B (1$\times$4) & CLIP-B32 & 0.80 & 0   & linear & $1.0 \times 10^{-5}$ \\
\addlinespace[3pt]
A-Left-Right (2$\times$2)    & CLIP-B32 & 0.25 & 0   & linear & $1.0 \times 10^{-5}$ \\
A-On-Under (2$\times$2)      & CLIP-B32 & 0.85 & 0   & linear & $1.0 \times 10^{-5}$ \\
B-Left-Right (2$\times$2)    & CLIP-B32 & 0.50 & 0   & cosine & $2.0 \times 10^{-5}$ \\
B-Front-Behind (2$\times$2)  & CLIP-B32 & 1.30 & 0   & cosine & $2.0 \times 10^{-5}$ \\
\bottomrule
\end{tabular}
\end{table}

\begin{table}[h]
\centering
  \caption{Hyperparameters used for experiments in \cref{sec:exp_global_matching}. 
We adopt percentile-based thresholding: at iteration $t$, the top $1-\tau_t$ fraction of pseudo-labels (ranked by similarity) is selected.}
\label{tab:hyperparams_global}
\setlength{\tabcolsep}{8pt}
\begin{tabular}{l l c c c c}
\toprule
Dataset & Model & $\tau_1$ & $\tau_T$ & Schedule & lr \\
\midrule
Winoground  & SigLIP-B16 & 0.50 & 0 & linear & $1.0 \times 10^{-5}$ \\
MMVP-VLM    & SigLIP-B16 & 0.55 & 0 & linear & $2.0 \times 10^{-5}$ \\
ColorSwap   & SigLIP-B16 & 0.50 & 0 & linear & $4.0 \times 10^{-5}$ \\
\bottomrule
\end{tabular}
\end{table}

\subsection{Complete results from \cref{sec:exp_without_metric_boost}}
\label{app:exp_results}

We present complete empirical results for \cref{fig:exp_1xk} below in \cref{tab:sugarcrepe,tab:whatsup}.
Following \citet{li2024exploring}, we further convert the \whatsup datasets into four directional variants with $2 \times 2$ group structures and present results in \cref{tab:whatsup_variants}: 
\cref{alg:ttm} again yields significant improvements---\textbf{up to 135.1\% relative gains and 95.5\% relative error reduction}---on top of \simplematchtext. 
Together, these results demonstrate that \ttmtext is broadly effective across both $k \times k$ and $1 \times k$ settings, even in cases where evaluation metrics themselves cannot induce gains.

\begin{table}[h]
\centering
  \caption{Performance on SugarCrepe datasets ($1 \times 2$ groups) without metric-induced boosts: for $1 \times k$ groups, \groupscoretext and \matchscoretext coincide. 
  Raw SigLIP-B16 performance is reported under \groupscoretext, 
and \ttmtext corresponds to the performance of \cref{alg:ttm}. 
  We report absolute gains ($\Delta$), relative gains, and relative error reductions of \ttmtext over the raw model performance.}
  \label{tab:sugarcrepe}
\begin{tabular}{l c c c c}
\toprule
{Datasets} & {\siglipbase} & {\ttmtext} & $\Delta$ & {Error Reduction} \\
\midrule
Replace Relation & 70.48 & \cellcolor{lightblue}\textbf{76.23\,\scriptsize$\pm$\,0.51} & \cellcolor{lightblue}{+}\,\textbf{5.8} \, (\textbf{8.2\%} $\uparrow$) & \cellcolor{lightblue}\textbf{19.5\%} $\downarrow$ \\
Swap Attribute  & 71.47 & \cellcolor{lightblue}\textbf{77.36\,\scriptsize$\pm$\,0.71} & \cellcolor{lightblue}{+}\,\textbf{5.9} \, (\textbf{8.2\%} $\uparrow$) & \cellcolor{lightblue}\textbf{20.6\%} $\downarrow$ \\
Swap Object     & 60.41 & \cellcolor{lightblue}\textbf{66.12\,\scriptsize$\pm$\,2.06} & \cellcolor{lightblue}{+}\,\textbf{5.7} \, (\textbf{9.5\%} $\uparrow$) & \cellcolor{lightblue}\textbf{14.4\%} $\downarrow$ \\
Add Attribute   & 83.67 & \cellcolor{lightblue}\textbf{88.95\,\scriptsize$\pm$\,0.83} & \cellcolor{lightblue}{+}\,\textbf{5.3} \, (\textbf{6.3\%} $\uparrow$) & \cellcolor{lightblue}\textbf{32.3\%} $\downarrow$ \\
\bottomrule
\end{tabular}
\end{table}

\begin{table}[h]
\centering
  \caption{Performance on \whatsup A/B datasets ($1 \times 4$ groups) without metric-induced boosts: for $1 \times k$ groups, \groupscoretext and \matchscoretext coincide. 
  Raw CLIP-B32 performance is reported under \groupscoretext, 
and \ttmtext corresponds to the performance of \cref{alg:ttm}. 
  We report absolute gains ($\Delta$), relative gains, and relative error reductions of \ttmtext over the raw model performance.}
  \label{tab:whatsup}
\begin{tabular}{l c c c c}
\toprule
  {Datasets} & {\clipbaselow} & {\ttmtext} & $\Delta$ & {Error Reduction} \\
\midrule
\whatsup A & 30.58 & \cellcolor{lightblue}\textbf{56.8\,\scriptsize$\pm$\,1.84} & \cellcolor{lightblue}{+}\,\textbf{26.2} \, (\textbf{85.7\%} $\uparrow$) & \cellcolor{lightblue}\textbf{37.7\%} $\downarrow$ \\
\whatsup B & 30.88 & \cellcolor{lightblue}\textbf{49.94\,\scriptsize$\pm$\,2.58} & \cellcolor{lightblue}{+}\,\textbf{19.1} \, (\textbf{61.7\%} $\uparrow$) & \cellcolor{lightblue}\textbf{27.6\%} $\downarrow$ \\
\bottomrule
\end{tabular}
\end{table}

\begin{table}[h]
\centering
  \caption{Performance on \whatsup $2 \times 2$ directional variants: LR: left-right, OU: on-under; FB: front-behind. 
Raw CLIP-B32 performance is reported under \groupscoretext. 
\simplematchtext corresponds to the performance under \matchscoretext (\cref{sec:metric}), 
and \ttmtext corresponds to the performance of \cref{alg:ttm}. 
  We report absolute gains ($\Delta$), relative gains, and relative error reductions of \ttmtext over \simplematchtext.}
  \label{tab:whatsup_variants}
\begin{tabular}{l c c c c c}
\toprule
  {Datasets} & {\clipbaselow} & {\simplematchtext} & {\ttmtext} & $\Delta$ & {Error Reduction} \\
\midrule
A-LR 
  & 0 
  & 40.78 
  & \cellcolor{lightblue}\textbf{95.87\,\scriptsize$\pm$\,4.42} 
  & \cellcolor{lightblue}{+}\,\textbf{55.1} \, (\textbf{135.1\%} $\uparrow$) 
  & \cellcolor{lightblue}\textbf{93.0\%} $\downarrow$ \\
A-OU 
  & 3.88 
  & 78.64 
  & \cellcolor{lightblue}\textbf{99.03\,\scriptsize$\pm$\,0} 
  & \cellcolor{lightblue}{+}\,\textbf{20.4} \, (\textbf{25.9\%} $\uparrow$) 
  & \cellcolor{lightblue}\textbf{95.5\%} $\downarrow$ \\
B-LR 
  & 0 
  & 55.88 
  & \cellcolor{lightblue}\textbf{82.84\,\scriptsize$\pm$\,0.49} 
  & \cellcolor{lightblue}{+}\,\textbf{27.0} \, (\textbf{48.2\%} $\uparrow$) 
  & \cellcolor{lightblue}\textbf{61.1\%} $\downarrow$ \\
B-FB 
  & 0 
  & 47.06 
  & \cellcolor{lightblue}\textbf{66.67\,\scriptsize$\pm$\,1.30} 
  & \cellcolor{lightblue}{+}\,\textbf{19.6} \, (\textbf{41.7\%} $\uparrow$) 
  & \cellcolor{lightblue}\textbf{37.0\%} $\downarrow$ \\
\bottomrule
\end{tabular}
\end{table}

\subsection{Runtime analysis and effectiveness of \ttmtext under a small number of iterations}
\label{app:runtime}

\paragraph{Runtime analysis.}
The runtime of TTM scales as $O(T \cdot C_{\ftmath})$, where $T$ is the number of iterations and $C_{\ftmath}$ denotes the model finetuning cost. While TTM also requires selecting the maximum-similarity matching within each group, this cost is dominated by the finetuning cost. Specifically:

\begin{itemize}
  \item 
The per-group matching cost is $P(k) = O(k^3)$ for $k \times k$ groups (via the Hungarian algorithm) or $P(k) = O(k)$ for $1 \times k$ groups. Since compositional benchmarks use very small group sizes ($k=2$ for $k \times k$ groups or $k=4$ for $1 \times k$ groups), we can safely treat $P(k) = O(1)$.
  \item Matching over all $n$ groups therefore costs $O(n)$, which is dominated by the model finetuning cost $C_{\ftmath}$.
\end{itemize}
Thus, the overall runtime is dominated by finetuning and scales as $O(T \cdot C_{\ftmath})$. In practice, the finetuning cost $C_{\ftmath}$ can be further reduced via efficient finetuning techniques \citep{hu2022lora}.

\paragraph{Effectiveness of \ttmtext under A small number of iterations.}
While we use $T=10$ in our main experiments with contrastive vision-language models, we also evaluate \ttmtext with $T=3$ and $T=5$. As shown \cref{tab:ttm_iterations}, \ttmtext continues to yield substantial gains even with only $T=3$ or $T=5$ iterations.

\begin{table}[h]
\centering
\caption{
  Performance on Winoground, MMVP-VLM, and ColorSwap under varying number of iterations ($T$).
Raw model performance is reported under \groupscoretext.
\simplematchtext corresponds to the performance under \matchscoretext (\cref{sec:metric}),
and \ttmtext corresponds to the performance of \cref{alg:ttm}.
We report absolute gains ($\Delta$), relative gains, and relative error reductions of \ttmtext over \simplematchtext.
}
\label{tab:ttm_iterations}
\begin{tabular}{l c c c c c}
\toprule
{Dataset (Iterations)} & {Raw} & {\simplematchtext} & {\ttmtext} & $\Delta$ & {Error Red.} \\
\midrule
\multicolumn{6}{l}{\textbf{Winoground}} \\
\quad SigLIP-B16 ($T=3$)  & 10.25 & 67.00 &
  \cellcolor{lightblue}\textbf{71.50\,\scriptsize$\pm$\,1.19} &
  \cellcolor{lightblue}{+}\,\textbf{4.5}\,(\textbf{6.7\%}\,$\uparrow$) &
  \cellcolor{lightblue}\textbf{13.6\%}\,$\downarrow$ \\
\quad SigLIP-B16 ($T=5$)  & 10.25 & 67.00 &
  \cellcolor{lightblue}\textbf{71.88\,\scriptsize$\pm$\,1.48} &
  \cellcolor{lightblue}{+}\,\textbf{4.9}\,(\textbf{7.3\%}\,$\uparrow$) &
  \cellcolor{lightblue}\textbf{14.8\%}\,$\downarrow$ \\
\quad SigLIP-B16 ($T=10$) & 10.25 & 67.00 &
  \cellcolor{lightblue}\textbf{72.50\,\scriptsize$\pm$\,0.64} &
  \cellcolor{lightblue}{+}\,\textbf{5.5}\,(\textbf{8.2\%}\,$\uparrow$) &
  \cellcolor{lightblue}\textbf{16.7\%}\,$\downarrow$ \\
\midrule
\multicolumn{6}{l}{\textbf{MMVP-VLM}} \\
\quad SigLIP-B16 ($T=3$)  & 22.96 & 81.48 &
  \cellcolor{lightblue}\textbf{85.19\,\scriptsize$\pm$\,0.91} &
  \cellcolor{lightblue}{+}\,\textbf{3.7}\,(\textbf{4.6\%}\,$\uparrow$) &
  \cellcolor{lightblue}\textbf{20.0\%}\,$\downarrow$ \\
\quad SigLIP-B16 ($T=5$)  & 22.96 & 81.48 &
  \cellcolor{lightblue}\textbf{87.04\,\scriptsize$\pm$\,1.11} &
  \cellcolor{lightblue}{+}\,\textbf{5.6}\,(\textbf{6.8\%}\,$\uparrow$) &
  \cellcolor{lightblue}\textbf{30.0\%}\,$\downarrow$ \\
\quad SigLIP-B16 ($T=10$) & 22.96 & 81.48 &
  \cellcolor{lightblue}\textbf{89.44\,\scriptsize$\pm$\,0.96} &
  \cellcolor{lightblue}{+}\,\textbf{8.0}\,(\textbf{9.8\%}\,$\uparrow$) &
  \cellcolor{lightblue}\textbf{43.0\%}\,$\downarrow$ \\
\midrule
\multicolumn{6}{l}{\textbf{ColorSwap}} \\
\quad SigLIP-B16 ($T=3$)  & 30.33 & 88.00 &
  \cellcolor{lightblue}\textbf{93.58\,\scriptsize$\pm$\,1.01} &
  \cellcolor{lightblue}{+}\,\textbf{5.6}\,(\textbf{6.3\%}\,$\uparrow$) &
  \cellcolor{lightblue}\textbf{46.5\%}\,$\downarrow$ \\
\quad SigLIP-B16 ($T=5$)  & 30.33 & 88.00 &
  \cellcolor{lightblue}\textbf{93.58\,\scriptsize$\pm$\,0.86} &
  \cellcolor{lightblue}{+}\,\textbf{5.6}\,(\textbf{6.3\%}\,$\uparrow$) &
  \cellcolor{lightblue}\textbf{46.5\%}\,$\downarrow$ \\
\quad SigLIP-B16 ($T=10$) & 30.33 & 88.00 &
  \cellcolor{lightblue}\textbf{94.25\,\scriptsize$\pm$\,0.43} &
  \cellcolor{lightblue}{+}\,\textbf{6.3}\,(\textbf{7.1\%}\,$\uparrow$) &
  \cellcolor{lightblue}\textbf{52.1\%}\,$\downarrow$ \\
\bottomrule
\end{tabular}
\vspace{-3pt}
\end{table}

\end{document}